\documentclass[journal,12pt,onecolumn,compsoc]{IEEEtran}

\usepackage[T1]{fontenc}
\usepackage[utf8]{inputenc}
\usepackage{lmodern}

\usepackage[pdftex]{graphicx}
\usepackage{color}
\usepackage[table,dvipsnames]{xcolor}
\usepackage{hhline}
\usepackage{longtable}

\usepackage[cmex10]{amsmath}
\usepackage{amssymb,amsthm}

\usepackage{algorithmic}
\usepackage[ruled,vlined]{algorithm2e}
\usepackage{adjustbox}
\usepackage{array}
\usepackage[caption=false,font=footnotesize]{subfig}
\usepackage{url}
\usepackage{enumitem}
\usepackage{multicol}
\usepackage{caption}
\usepackage[bb=boondox]{mathalfa}

\theoremstyle{plain}
\newtheorem{theorem}{Theorem}[section]
\newtheorem{proposition}[theorem]{Proposition}

\newtheorem{lemma}[theorem]{Lemma}

\theoremstyle{remark}
\newtheorem{remark}[theorem]{Remark}

\newtheorem{assumption}[theorem]{Assumption}

\theoremstyle{definition}
\newtheorem{definition}[theorem]{Definition}

\usepackage{pgf}
\usepackage{tikz}
\usetikzlibrary{calc,fadings,decorations.pathreplacing}
\usetikzlibrary{matrix,arrows}
\usepackage{collcell}
\usepackage{xstring}

\def\minus{%
  \setbox0=\hbox{-}%
  \vcenter{%
    \hrule width\wd0 height \the\fontdimen8\textfont3%
  }%
}

\newcommand{\vect}[1]{\boldsymbol{#1}}
\newcommand{\node}[1]{\mathrm{#1}}

\newcommand{\ds}{\displaystyle}
\newcommand{\ts}{\textstyle}

\newcommand{\Nn}{{\mathbb N}}

\newcommand{\Rr}{{\mathbb R}}

\newcommand{\psiGBF}{$\vect{\psi}$-GBF}
\newcommand{\psiGBFRLS}{$\vect{\psi}$-GBF-RLS}

\newcommand{\normconst}{\|(\mathbf{S}_W \mathbf{B}_M\!)^{\minus 1}\|  }
\newcommand{\itpx}{y_{\vect{\psi}}^{\circ}}
\newcommand{\itpy}{y^{\circ}}
\newcommand{\funit}{f_{\mathbb{1}}}

\newcommand{\sgn}{\mathrm{sign}}

\newcommand{\Ll}{\mathbf{L}}
\newcommand{\Aa}{\mathbf{A}}
\newcommand{\Dd}{\mathbf{D}}
\newcommand{\Kk}{\mathbf{K}}

\newcommand{\bin}{\hspace{-2pt}\mbox{\resizebox{0.9em}{0.3em}{ 
\rule{9pt}{5pt}\hspace{-8pt}{\color{white}%
\rule[2pt]{3pt}{1pt}\hspace{-3pt} 
\rule[2pt]{3pt}{1pt}\hspace{-2pt}\rule[1pt]{1pt}{3pt} 
}}}\hspace{-2pt}}

\newcommand{\SIM}{\mathrel{\hss\hbox{\scalebox{0.4}{$\mathrm{SIM}$}}\hss}}

\newcommand\twomoon{%
  \mathrel{\ooalign{\hss\raise0.4ex\hbox{\scalebox{0.6}{$\cap$}}
  \kern-1.4ex\hbox{\scalebox{0.6}{$\cup$}}\hss\cr}}}

\newcommand*{\MinNumber}{0.0}%
\newcommand*{\MidNumber}{0.65} %
\newcommand*{\MaxNumber}{1.0}%

\newcommand{\ApplyGradient}[1]{%
        \ifdim #1 pt > \MidNumber pt
            \pgfmathsetmacro{\PercentColor}{max(min(100.0*(#1 - \MidNumber)/(\MaxNumber-\MidNumber),100.0),0.00)} %
            \hspace{-0.33em}\colorbox{Purple!\PercentColor!white}{#1}
        \else
            \pgfmathsetmacro{\PercentColor}{max(min(100.0*(\MidNumber - #1)/(\MidNumber-\MinNumber),100.0),0.00)} %
            \hspace{-0.33em}\colorbox{white!\PercentColor!white}{#1}
            
        \fi
}

\newcommand{\ApplyG}[2]{%
        \ifdim #1 pt > \MidNumber pt
            \pgfmathsetmacro{\PercentColor}{max(min(100.0*(#1 - \MidNumber)/(\MaxNumber-\MidNumber),100.0),0.00)} %
            \hspace{-0.33em}\colorbox{Purple!\PercentColor!white}{#2}
        \else
            \pgfmathsetmacro{\PercentColor}{max(min(100.0*(\MidNumber - #1)/(\MidNumber-\MinNumber),100.0),0.00)} %
            \hspace{-0.33em}\colorbox{white!\PercentColor!white}{#2}
            
        \fi
}

\newcolumntype{G}{>{\collectcell\ApplyGradient}c<{\endcollectcell}}
\newcolumntype{R}[2]{%
    >{\adjustbox{angle=#1,lap=\width-(#2)}\bgroup}%
    l%
    <{\egroup}%
}
\newcommand*\rz{\multicolumn{1}{R{0}{-2em}}}

\makeatletter

\newcommand\coltab[1]{\@coltab#1\q@stop}

\long\def\@coltab#1(#2)\q@stop{%
    \ApplyG{#1}{#1 ($\pm$ #2)}
}
\makeatother

\newcolumntype{T}{>{\collectcell\coltab}l<{\endcollectcell}}

\begin{document}

\title{Semi-Supervised Learning on Graphs with Feature-Augmented Graph Basis Functions
}

\author{Wolfgang Erb
\thanks{Universit{\`a} degli Studi di Padova, Dipartimento di Matematica ''Tullio Levi-Civita'', wolfgang.erb@lissajous.it.}
}

\markboth{Semi-Supervised Learning on graphs with Feature-Augmented GBF's}%
{Semi-Supervised Learning on graphs with Feature-Augmented GBF's}

\maketitle

\begin{abstract}
\noindent For semi-supervised learning on graphs, we study how 
initial kernels in a supervised learning regime can be augmented with additional information from known priors or from unsupervised learning outputs. These augmented kernels are constructed in a simple update scheme based on the Schur-Hadamard product of the kernel with additional feature kernels. As generators of the positive definite kernels we will focus on graph basis functions (GBF) that allow to include geometric 
information of the graph via the graph Fourier transform. Using a regularized least squares (RLS) approach for machine learning, we will test the derived augmented kernels for the classification of data on graphs. 
\end{abstract}


\IEEEpeerreviewmaketitle

\section{Introduction}

\IEEEPARstart{S}{emi-supervised} learning (SSL) methods are devised to incorporate additional information of unlabeled data into a learning task.
Compared to a purely supervised regime, SSL needs less labeled data and potentially increases the accuracy of learning \cite{Zhu05}. On the other hand, it requires reliable prior knowledge of the underlying structure of the data. The challenging part in the design of SSL algorithms is therefore to find smart, robust and efficient ways to combine prior sources with the information of the labels. 

For the classification of data on graphs, semi-supervised methods are particularly promising. In this setting, not only attributes on the nodes, but also neighborhood relations between the vertices, and the global geometric structure of the graph can be incorporated in the learning scheme. Further, graph clustering algorithms provide powerful unsupervised classification schemes that are able to extract essential substructures of the graph and, in this way, to generate additional priors for semi-supervised learning schemes. 

In this work, we are interested in kernel-based SSL models in which the nodes of a graph  are classified by the solution of a regularized least squares (RLS) problem. The regularization is determined by a reproducing kernel Hilbert space norm forcing the solution to be smoothly representable in terms of the given kernel. The kernel itself can be chosen flexibly according to the given data. Therefore, a smart strategy to integrate the geometric graph structure and additional feature information into the model kernel increases the classification performance of the corresponding SSL scheme. 

Goal of this work is to present a simple update strategy in which kernels on graphs are augmented with feature information. This easy to implement update procedure allows us to include geometric similarities on the graph, relations between node attributes as well as additional priors derived from unsupervised schemes. The resulting augmented kernels need less labels than the initial kernels and enhance the classification accuracy.  

As generators of the kernels, we will focus on positive definite graph basis functions (GBF's). GBF's are graph analogs of radial basis functions in $\Rr^d$ \cite{erb2019b} that generate kernels in terms of generalized shifts of a principal basis function. More important, based on the graph Fourier transform, GBF's give a compact description of the involved reproducing kernel Hilbert spaces. Further, as indicated in Fig. \ref{fig:illustrationaugmentedGBF}, the usage of GBF's allows to describe the kernel augmentation step in a simple way.

\begin{figure}[htb]
\begin{equation*}
	 \underbrace{\vcenter{\hbox{\includegraphics[width = 4cm]{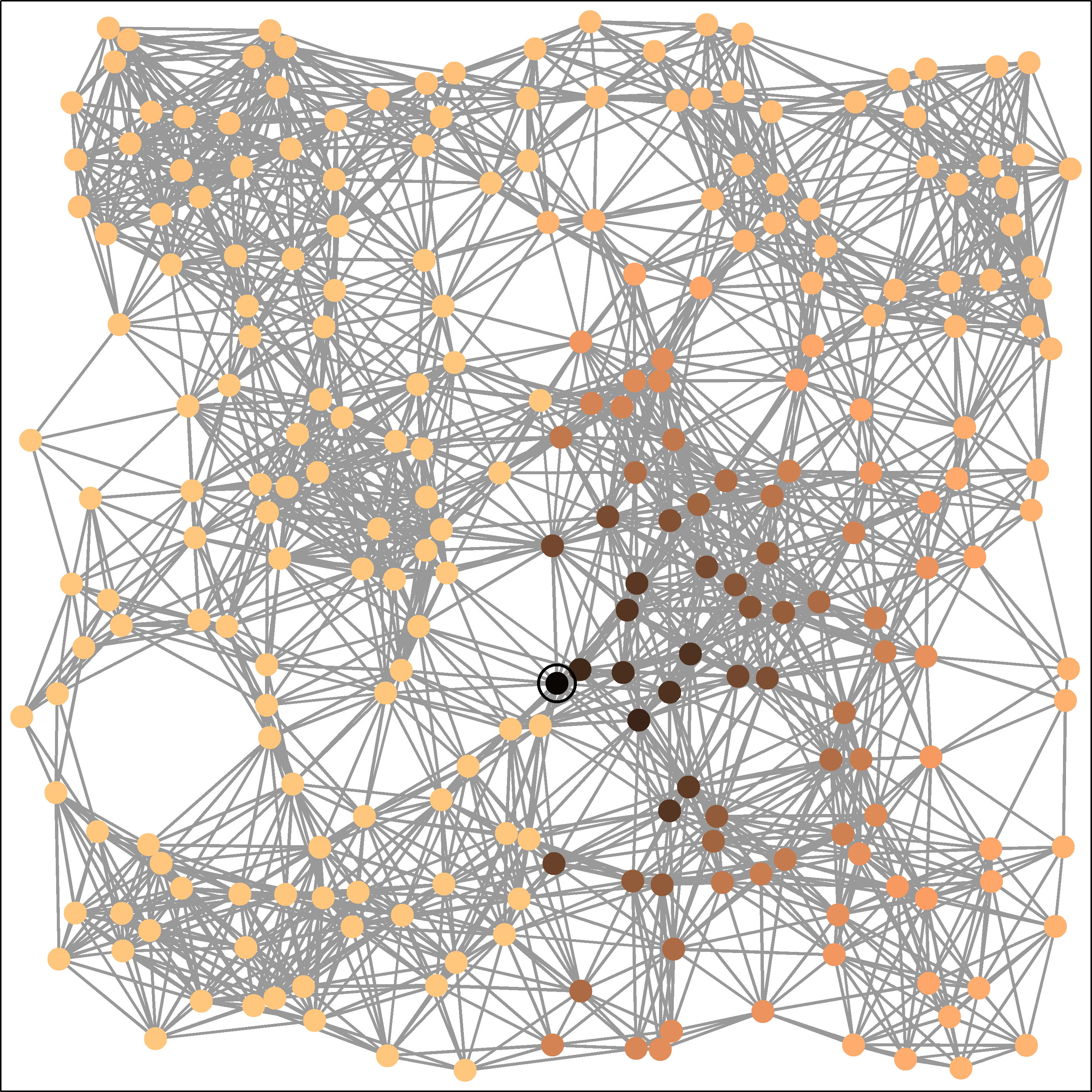}}}}_{\begin{minipage}{3.5cm} \centering \small Feature-augmented \psiGBF{} $K_{\vect{\psi}}(\cdot,\node{w})$ \end{minipage}}
	 \; = \;	
	 \underbrace{\vcenter{\hbox{\includegraphics[width = 4cm]{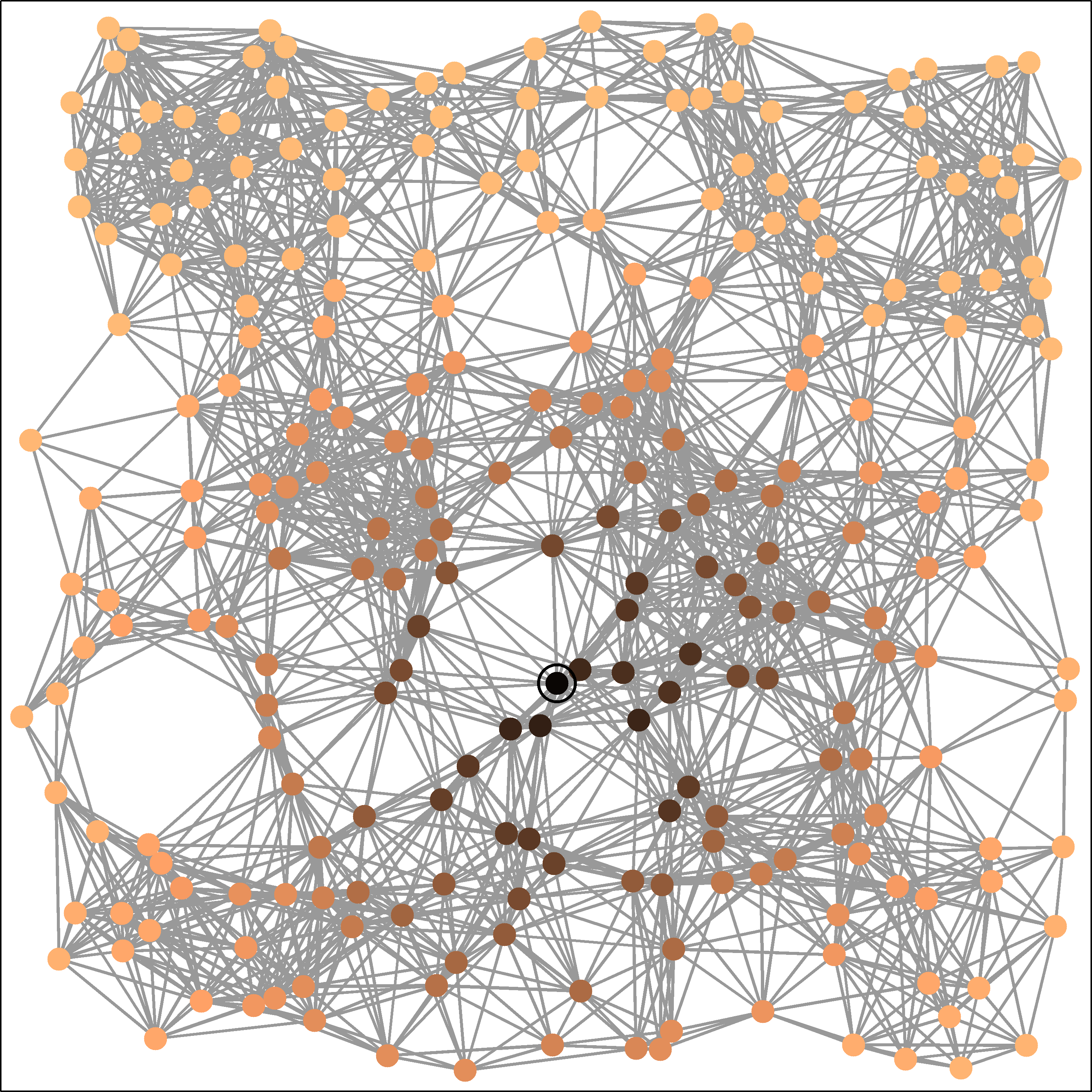}}}}
_{\begin{minipage}{3.5cm} \centering \small Diffusion GBF $ \mathbf{C}_{\delta_{\node{w}}}f $ \end{minipage}}
	 \; \odot \; 
		 \underbrace{\vcenter{\hbox{\includegraphics[width = 4cm]{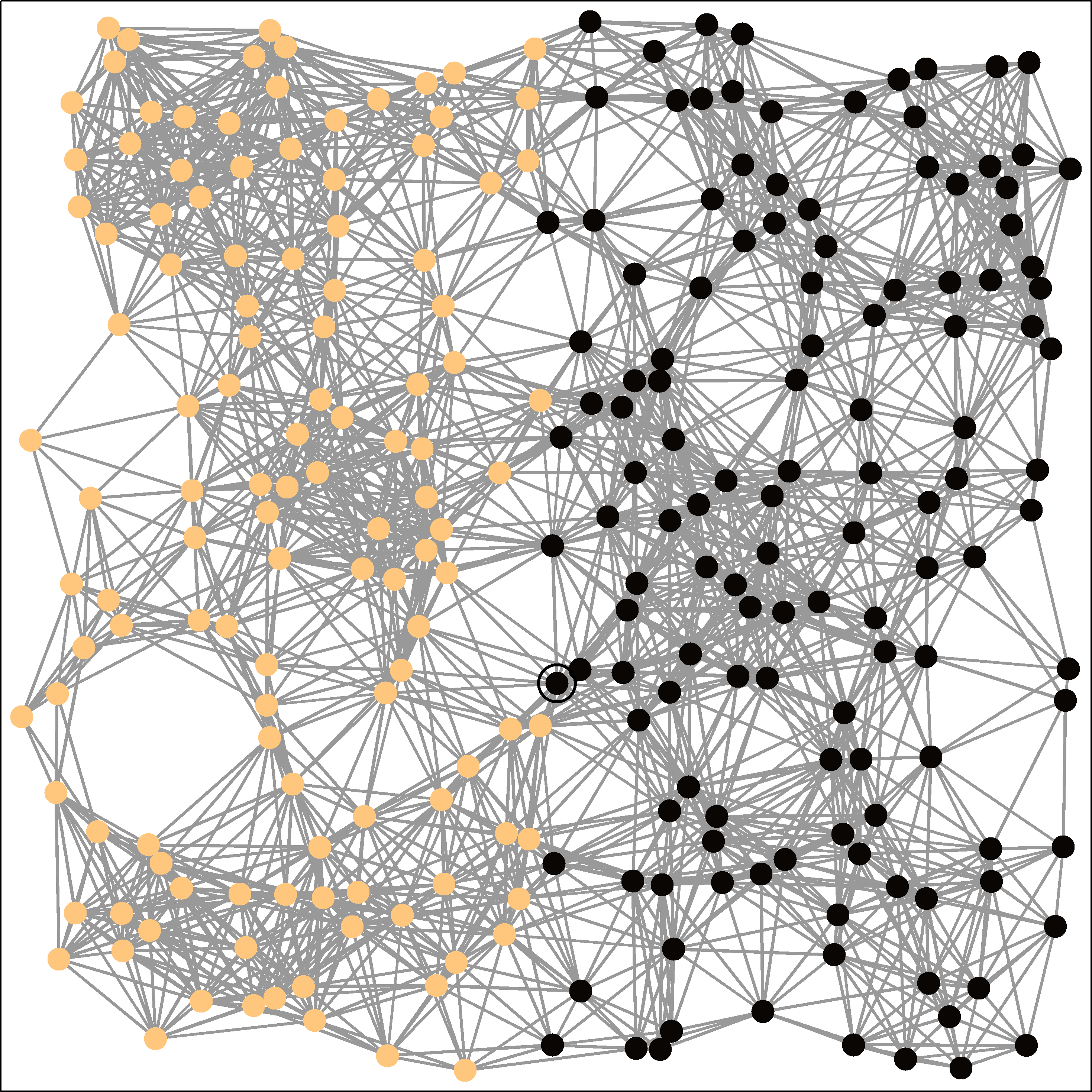}}}}
_{\begin{minipage}{3.5cm} \centering \small Feature information $ \mathbf{C}_{\delta_{\psi(\node{w})}}f^F \circ \psi $ \end{minipage}}
\end{equation*}
	\caption{Illustration on how a feature-adapted basis function $K_{\vect{\psi}}(\cdot,\node{w})$ on a sensor graph can be constructed as the Schur-Hadamard product of a graph basis function $\mathbf{C}_{\delta_{\node{w}}}f$ with additional feature information. The ringed node indicates the center node $\node{w}$.}
  	 \label{fig:illustrationaugmentedGBF}
\end{figure}

\noindent \textbf{Construction principle for the augmented kernels.} 
The initial building blocks of our SSL model are kernels incorporating similarities between graph nodes. Such kernels can be constructed easily in terms of the graph Laplacian or, more generally, by positive definite GBF's as generators. 
The information of the additional features is provided by a set of feature maps and feature graphs. As for the graph itself, the similarities on the feature graphs are captured by respective feature kernels. 

As soon as this information is available, we construct a tensor-product kernel on the Cartesian product of the graph with the feature graphs. The given feature maps provide an embedding of the graph into the product graph. In a central step of the construction, this embedding allows us to extract a feature-augmented kernel from the tensor-product kernel. An intriguing detail of this construction is the fact that the augmented kernel can be formulated as the Schur-Hadamard product of the original kernel with an additional feature matrix, as illustrated in Fig. \ref{fig:illustrationaugmentedGBF}. This allows us to formulate the entire kernel augmentation procedure with simple update steps based on the Schur-Hadamard product. 
We present the details of this construction in Section \ref{sec:construction}.

\vspace{2mm}   

\noindent \textbf{Outline of the paper.} 
The required terminologies for spectral graph theory and kernel methods are derived in Section \ref{sec:spectraltheory} and Section \ref{sec:kernelmethods}, respectively. The necessary background for the Cartesian products of graphs and the involved tensor-product construction is given in Section \ref{sec:cartesianproduct}. The main construction of the augmented kernels is presented in Section \ref{sec:psiGBFRLS}. In particular, it is shown how, in a kernel-based setting for SSL, initial kernels on a graph can be augmented with additional information provided by feature maps. In Section
\ref{sec:featuremaps}, we give some examples on how feature maps and feature graphs can be generated and incorporated in the SSL scheme. Further, for binary feature maps we provide a more profound analysis of the augmented kernels and the corresponding reproducing kernel Hilbert spaces (Section \ref{sec:augmentationbinary}). Finally in Section \ref{sec:numericalexamples}, we give several experiments that show how augmented graph basis functions can be applied to particular classification tasks on graphs. 

\vspace{2mm}   

\noindent \textbf{Literature.} 
The amount of literature on SSL is huge and, regrettably, we can not mention all of it. In the following, we will focus on the most important influences for this work. The mathematical foundations of machine learning involving kernel techniques can be found in the monographs \cite{Schoelkopf2002,Vapnik1998}. A general overview for different SSL methods on graphs is given in \cite{Zhu05}.

\begin{itemize}
\item[-] Kernel techniques on graphs linked to the spectral decomposition of the graph Laplacian were studied in \cite{KondorLafferty2002,SmolaKondor2003}. In this work, we will use a related kernel concept based on the generalized shifts of a graph basis function (GBF) \cite{erb2019b}.
\item[-] The presented feature-augmented kernels can be considered as graph-based constructions of variably scaled kernels. These kernels were developed for adaptive interpolation \cite{Bozzini1,demarchi2019a,demarchi2019b} and learning \cite{campi2019} in the euclidean setting. An important conceptual difference in this article is the involved tensor-product construction that allows to obtain the augmented kernels by a simple update. 
\item[-] For supervised classification, we will follow a regularized least squares (RLS) approach as for instance framed in \cite{Rifkin2003}. Particularly for graphs, similar kernel-based classification methods have been considered, for instance, in \cite{BelkinMatveevaNiyogi2004} and \cite{Romero2017}.
\item[-] The SSL scheme considered in this work is inherently transductive in the sense that the obtained classification can not be expanded beyond the given graph structure. This is in contrast to inductive SSL schemes as for instance studied in \cite{Belkin2006,BelkinNiyogi2004} where graphs are embedded in manifold structures. Other well-known transductive SSL schemes on graphs are for instance transductive SVM's \cite{Joachims1999,Vapnik1998} or Naive Bayes approaches \cite{Nigam2000}.

\item[-] A large line of research in kernel-based machine learning is related to the extraction of data-driven optimal kernels from a given family of kernels \cite{goenen2011}. Such multiple kernel learning or extraction strategies can naturally be included in different types of SSL schemes \cite{Aiolli2015,Mhaskar2019,Wang2012,Yan2010}. The approach of the actual work is simpler. Here, the augmented kernel is constructed solely by multiplicative updates based on additional feature information. 

\end{itemize}

\vspace{2mm}

\section{Background} \label{sec:spectraltheory}
\subsection{Spectral Graph Theory} \label{sec:spectralgraphtheory}
We give a short introduction to graph theory and the notion of spectrum and convolution on a graph. A standard reference for spectral graph theory is \cite{Chung}, an introduction to the graph Fourier transform and space-frequency concepts is given in \cite{shuman2016}. 

We will regard a graph $G$ as a triplet $G=(V,E,\mathbf{L})$ consisting of a finite set $V=\{\node{v}_1, \ldots, \node{v}_{n}\}$ of vertices, a set $E \subseteq V \times V$ of edges connecting the vertices and a graph Laplacian $\Ll \in \Rr^{n \times n}$. We understand $\mathbf{L}$ as a generalized graph Laplacian (see \cite[Section 13.9]{GodsilRoyle2001} in the sense that $\Ll$ is a symmetric matrix and the entries
$\Ll_{i,j}$ satisfy 
\begin{equation} \label{eq:generalizedLaplacian}
\ds {\begin{array}{ll}\; \Ll_{i,j}<0& \text{if $i \neq j$ and $\node{v}_{i}, \node{v}_{j}$ are connected}, \\ \; \Ll_{i,j}=0 & \text{if $i\neq j $ and $\node{v}_{i}, \node{v}_{j} $ are not connected}, \\ \; \Ll_{i,i} \in \Rr & \text{for $i \in \{1, \ldots, n\}$}.\end{array}}\end{equation}
In general, the negative non-diagonal elements of the Laplacian $\Ll$ describe connection weights of the edges, while the diagonal elements can be used to differentiate the importance of the single vertices. Important examples of $\Ll$ are: 
\begin{itemize} 
\item[(1)] $\Ll_A = - \mathbf{A}$, where $\mathbf{A}$ denotes the adjacency matrix of the graph given by
\begin{equation*}
    \mathbf{A}_{i,j} := 
  \begin{cases}
    1, & \text{if $i \neq j$ and $\node{v}_i, \node{v}_j$ are connected}, \\
    0, & \text{otherwise},
  \end{cases}.
\end{equation*}
\item[(2)] $\Ll_S = \mathbf{D} - \mathbf{A}$, where $\mathbf{D}$ is the degree matrix with the entries given by
\begin{equation*}
    \mathbf{D}_{i,j} := 
  \begin{cases}
    \sum_{k=0}^n \mathbf{A}_{i,k}, & \text{if } i=j, \\
    0, & \text{otherwise}.
  \end{cases}
\end{equation*}
In algebraic graph theory, $\Ll_S$ is the most common definition for a graph Laplacian. The matrix $\Ll_S$ is positive definite.  
\item[(3)] $\Ll_N = \mathbf{D}^{-\frac{1}{2}} \Ll_S \mathbf{D}^{-\frac{1}{2}} = \mathbf{I}_n - \mathbf{D}^{-\frac{1}{2}} \mathbf{A} \mathbf{D}^{-\frac{1}{2}}$ is called the normalized graph Laplacian of $G$. Here, $\mathbf{I}_n$ denotes the identity matrix in $\Rr^n$. One particular feature of the normalized
graph Laplacian $\Ll_N$, is the fact that its spectrum is contained in the interval $[0,2]$, see \cite[Lemma 1.7]{Chung}.
\end{itemize}

\subsection{Fourier Transform of Graph Signals} 
We denote the vector space of all real-valued
signals $x: V \rightarrow \mathbb{R}$ on $G$ as $\mathcal{L}(G)$. Since $G$ consists of $n$ nodes, the dimension of $\mathcal{L}(G)$ is exactly $n$. As the node set $V$ is ordered, we can describe every signal $x$ also as a vector 
$x = (x(\node{v}_1), \ldots, x(\node{v}_n))^{\intercal}\in \mathbb{R}^n$. Depending on the context, we will switch between the representation of $x$ as a function in $\mathcal{L}(G)$ and a vector in $\Rr^n$. On the space $\mathcal{L}(G)$, we have a natural inner product given by $$y^\intercal x := \sum_{i=1}^n x(\node{v}_i) y(\node{v}_i).$$ The corresponding euclidean norm is given by $\|x\|^2 := x^{\intercal} x = \sum_{i=1}^n x(\node{v}_i)^2$. The canonical orthonormal basis in $\mathcal{L}(G)$ is given by the system $\{\delta_{\node{v}_1}, \ldots, \delta_{\node{v}_n}\}$ where the unit vectors $\delta_{\node{v}_j}$ satisfy
$\delta_{\node{v}_j}(\node{v}_i) = \delta_{i,j}$ for $i,j \in \{1, \ldots,n\}$.

The harmonic structure on $G$ is determined by the graph Laplacian $\Ll$. As $\Ll$ is symmetric, this harmonic structure does not depend on the orientation of the edges in $G$. The graph Laplacian $\Ll$ allows to introduce a graph Fourier transform on $G$ in terms of the orthonormal eigendecomposition
\begin{equation*}
\mathbf{L}=\mathbf{U}\mathbf{M}_{\lambda} \mathbf{U^\intercal}.
\end{equation*}
Here,
$\mathbf{M}_{\lambda} = \mathrm{diag}(\lambda) = \text{diag}(\lambda_1,\ldots,\lambda_{n})$ 
denotes the diagonal matrix with the increasingly ordered eigenvalues $\lambda_i$, $i \in \{1, \ldots, n\}$, of $\mathbf{L}$ as diagonal entries.
The columns $ u_1, \ldots, u_{n}$ of the orthonormal matrix $\mathbf{U}$ are normalized eigenvectors of
$\mathbf{L}$ with respect to the eigenvalues $\lambda_1, \ldots, \lambda_n$. The ordered set $\hat{G} = \{u_1, \ldots, u_{n}\}$ of eigenvectors is an orthonormal basis for the space $\mathcal{L}(G)$ of signals on the graph $G$. We call $\hat{G}$ the spectrum of the graph $G$. 

In classical Fourier analysis, as for instance the Euclidean space or the torus, the Fourier transform can be defined in terms of the eigenvalues and eigenfunctions of the Laplace operator.
In analogy, we consider the elements of $\hat{G}$, i.e. the eigenvectors $\{u_1, \ldots, u_{n}\}$, as the Fourier basis on $G$. In particular, going back to our 
spatial signal $x$, we can define the graph Fourier transform of $x$ as the vector
\begin{equation*}
\hat{x} := \mathbf{U^\intercal}x  = (u_1^\intercal x, \ldots, u_n^\intercal x)^{\intercal},
\end{equation*}
with the inverse graph Fourier transform given as
\begin{equation*}
x = \mathbf{U}\hat{x}. 
\end{equation*}
The entries $\hat{x}_i = u_i^\intercal x$ of $\hat{x}$ are the frequency components or coefficients
of the signal $x$ with respect to the basis functions $u_i$. For this reason, $\hat{x} : \hat{G} \to \Rr$ can be regarded as a function on the spectral domain $\hat{G}$ of the graph $G$. 
To keep the notation simple, we will however usually represent spectral distributions $\hat{x}$ as 
vectors $(\hat{x}_1, \ldots, \hat{x}_n)^{\intercal}$ in $\Rr^n$.   

\subsection{Convolution and the graph $C^{\ast}$-algebra}

Based on the graph Fourier transform we can introduce a convolution operation between two graph signals $x$ and $y$. For this, we use an analogy to the convolution theorem in classical Fourier analysis linking convolution in the spatial domain to 
pointwise multiplication in the Fourier domain. In this way, the graph convolution for two signals $x,y \in \mathcal{L}(G)$ can be defined as
\begin{equation}
    x \ast y := \mathbf{U} \left ( \mathbf{M}_{\hat{x}} \hat{y} \right ) = \mathbf{U}\mathbf{M}_{\hat{x}}\mathbf{U^\intercal}y \label{eq:spectralfilter1}.
\end{equation}
As before, $\mathbf{M}_{\hat{x}}$ denotes the diagonal matrix 
$\mathbf{M}_{\hat{x}} = \mathrm{diag}(\hat{x})$ and $\mathbf{M}_{\hat{x}} \hat{y} = (\hat{x}_1 \hat{y}_1, \ldots, \hat{x}_{n} \hat{y}_{n})$
gives the pointwise product of the two vectors $\hat{x}$ and $\hat{y}$. 
The convolution $\ast$ on $\mathcal{L}(G)$ has the following properties:
\begin{itemize}
\item[(i)] $x \ast y = y \ast x$ (Commutativity),
\item[(ii)] $(x \ast y) \ast z = x \ast (y \ast z)$ (Associativity),
\item[(iii)] $(x + y) \ast z = x \ast z + (y \ast z)$ (Distributivity),
\item[(iv)] $(\alpha x) \ast y = \alpha(y \ast x)$ for all $\alpha \in \Rr$ (Associativity for scalar multiplication).
\end{itemize}
The unity element of the convolution is given by $\funit = \sum_{i=1}^n u_i$. 
In view of the linear structure in equation \eqref{eq:spectralfilter1}, we can further define a convolution operator $\mathbf{C}_x$ on $\mathcal{L}(G)$ as
\[\mathbf{C}_x = \mathbf{U}\mathbf{M}_{\hat{x}}\mathbf{U^\intercal}. \]
Written in this way, $x \ast y$ corresponds to the matrix-vector product $\mathbf{C}_x y = x \ast y$, and we can regard every $x \in \mathcal{L}(G)$ also as a filter function acting by convolution on a second signal $y$. 

The rules (i)-(iv) of the graph convolution ensure that the vector space $\mathcal{L}(G)$ endowed with the convolution $\ast$ as a multiplicative operation is a commutative and associative algebra. With the identity as a trivial involution and the norm 
$$\|x\|_{\mathcal{A}} = \sup_{\|y\| = 1} \|x \ast y\|$$
we obtain a real $C^{\ast}$-algebra $\mathcal{A}$. This graph $C^{\ast}$-algebra $\mathcal{A}$ can be considered as a standard model for graph signal processing. It contains all possible signals and filter functions on $G$ and describes how filters act on signals via convolution. Furthermore, $\mathcal{A}$ contains the entire information of the graph Fourier transform, see \cite{erb2019b}. 

\section{Regularized least squares (RLS) classification with positive definite kernels} \label{sec:kernelmethods}

We give a short synthesis of well-known facts about regularized least squares (RLS) methods for classification. A historical overview for RLS in machine learning and a comparison to support vector machines is given in \cite{Rifkin2003}. A more general introduction to kernel-based methods for machine learning can be found in \cite{Schoelkopf2002,Vapnik1998}.  

\subsection{Positive definite kernels}
We consider symmetric and positive definite kernels $K : V \times V \to \Rr$ on the vertex set $V$. Linked to the kernel $K$ is a linear operator $\mathbf{K}: \mathcal{L}(G) \to \mathcal{L}(G)$ acting on a graph signal $x \in\mathcal{L}(G)$ by
\[\mathbf{K} x(\node{v}_i) = \sum_{j=1}^n K(\node{v}_i,\node{v}_j) x(\node{v}_j).\] 
By identifying signals $x \in \mathcal{L}(G)$ with vectors in 
$\Rr^n$, we can represent $\mathbf{K}$ as the $n \times n$-matrix
\[ \mathbf{K} = \begin{pmatrix} K(\node{v}_1,\node{v}_1) & K(\node{v}_1,\node{v}_2) & \ldots & K(\node{v}_1,\node{v}_n) \\
K(\node{v}_2,\node{v}_1) & K(\node{v}_2,\node{v}_2) & \ldots & K(\node{v}_2,\node{v}_n) \\
\vdots & \vdots & \ddots & \vdots \\
K(\node{v}_n,\node{v}_1) & K(\node{v}_n,\node{v}_2) & \ldots & K(\node{v}_n,\node{v}_n)
\end{pmatrix}.\]
In this way, the notion of positive definiteness can be transferred from $\Kk$ to the kernel $K$. 
\begin{definition} \label{def:pd}
A kernel $K$ is called positive definite (p.d.) 
if the matrix $\mathbf{K} \in \Rr^{n \times n}$ is symmetric and positive definite, i.e., we have $\mathbf{K}^\intercal = \mathbf{K}$ and $x^{\intercal} \mathbf{K} x > 0$ for all $x \in \Rr^n$, $x \neq 0$. Correspondingly, $K$ is called positive semi-definite (p.s.d.) if 
$\mathbf{K} \in \Rr^{n \times n}$ is symmetric and $x^{\intercal} \mathbf{K} x \geq 0$ for all $x \in \Rr^n$.
\end{definition}
\subsection{RLS solutions in reproducing kernel Hilbert spaces} \label{subsec:kernelRLS}
With a p.d. kernel $K$ we can define an inner product $\langle x,y \rangle_{K}$ and a norm $\| x \|_{K}$ as
\[ \langle x,y \rangle_{K} = y^{\intercal} \mathbf{K}^{-1} x, \quad 
\|x\|_K = \sqrt{\langle x,x \rangle_{K}}, \qquad x,y \in \mathcal{L}(G).\]
The space $\mathcal{L}(G)$ of signals endowed with this inner product is a reproducing kernel Hilbert space $\mathcal{N}_{K}$ (a systematic study is given in \cite{Aronszajn1950}) in which every signal $x \in \mathcal{L}(G)$ can be recovered from  $K$ as 
\[ \langle x, K(\cdot,\node{v}) \rangle_{K} = x^{\intercal} \mathbf{K}^{-1} K(\cdot,\node{v}) = x(\node{v}).\]
A RLS problem can be formulated in terms of the native space $\mathcal{N}_{K}$ and the kernel $K$. We call $y^*$ a RLS solution if it minimizes the regularized least-squares functional
\begin{equation} \label{eq:RLSfunctional}
y^* = \underset{x \in \mathcal{N}_{K}}{\mathrm{argmin}} \left( \frac{1}{N} \sum_{i=1}^N |x(\node{w}_i)-y(\node{w}_i)|^2 + \gamma \|x\|_{K}^2 \right), \quad \gamma > 0.
\end{equation}
The values $y(\node{w}_i) \in \Rr$, $i \in \{1, \ldots, N\}$ are given data values on a fixed subset $W = \{\node{w}_1, \ldots, \node{w}_N\} \subset V$ that we want to approximate with the RLS solution $y^*$.
The representer theorem \cite[Theorem 4.2]{Schoelkopf2002} states that the minimizer $y^*$ of the RLS functional
can be uniquely expressed as a linear combination
\begin{equation} \label{eq:representertheorem}
y^*(\node{v}) = \sum_{i = 1}^N c_i K(\node{v},\node{w}_i).
\end{equation}
It is well-known (see \cite{Rifkin2003}, \cite[Theorem 1.3.1.]{Wahba1990}) that the coefficients $c_i$ in the representation \eqref{eq:representertheorem} can be calculated as the solution of the linear system 
\begin{equation} \label{eq:computationcoefficients} 
 \left(\begin{array}{c} \phantom{c_1} \\ \phantom{c_2} \\ \phantom{\vdots} \\ \phantom{c_N} \end{array}\right. \hspace{-0.9cm} \underbrace{ \begin{pmatrix} K(\node{w}_1,\node{w}_1) & K(\node{w}_1,\node{w}_2) & \ldots & K(\node{w}_1,\node{w}_N) \\
K(\node{w}_2,\node{w}_1) & K(\node{w}_2,\node{w}_2) & \ldots & K(\node{w}_2,\node{w}_N) \\
\vdots & \vdots & \ddots & \vdots \\
K(\node{w}_N,\node{w}_1) & K(\node{w}_N,\node{w}_2) & \ldots & K(\node{w}_N,\node{w}_N)
\end{pmatrix}}_{\mathbf{K}_W}  + \gamma N \mathbf{I}_N \hspace{-1cm} \left. \begin{array}{c} \phantom{c_1} \\ \phantom{c_2} \\ \phantom{\vdots} \\ \phantom{c_N} \end{array}\right) \begin{pmatrix} c_1 \\ c_2 \\ \vdots \\ c_N \end{pmatrix}
= \begin{pmatrix} y(\node{w}_1) \\ y(\node{w}_2) \\ \vdots \\ y(\node{w}_N) \end{pmatrix}.
\end{equation}
With $\mathbf{K}$ being p.d. also the submatrix $\mathbf{K}_W$ is p.d. by the inclusion principle \cite[Theorem 4.3.15]{HornJohnson1985}. The linear system \eqref{eq:computationcoefficients} is therefore uniquely solvable. By \eqref{eq:representertheorem}, the RLS solution $y^*$ 
can be uniquely written in terms of $\{K(\cdot,\node{w}_1), \ldots, K(\cdot,\node{w}_N)\}$. 
We denote the corresponding approximation space as
\[\mathcal{N}_{K,W} = \left\{x \in \mathcal{L}(G) \ | \ x(\node{v}) = \sum_{k=1}^N c_k K(\node{v},\node{w}_k)\right\}. \]
For a vanishing regularization parameter $\gamma \to 0$, the limit $\itpy = \lim_{\gamma \to 0}y^*$ is uniquely determined by the condition \eqref{eq:representertheorem} and the coefficients calculated in \eqref{eq:computationcoefficients} with $\gamma = 0$. 
The resulting signal $\itpy$ interpolates the data $(\node{w}_i,y(\node{w}_i))$, i.e. we have $\itpy(\node{w}_i) = y(\node{w}_i)$ for all $i \in \{1, \ldots, N\}$. 

\subsection{Positive definite functions on graphs and GBF-RLS solutions} \label{sec:pdfunctions}
Using positive definite functions on $G$ we can precisely encode those p.d. kernels that have a Mercer decomposition in terms of the Fourier basis on $G$ \cite{erb2019b}. 

\begin{samepage}
\begin{definition} \label{def:pdfunction}
A function $f: V \to \Rr$ on $G$ is called a positive definite (positive semi-definite) graph basis function (GBF) if the matrix 
\[ \mathbf{K}_{f} = \begin{pmatrix} \mathbf{C}_{\delta_{\node{v}_1}} f(\node{v}_1) & \mathbf{C}_{\delta_{\node{v}_2}} f(\node{v}_1) & \ldots & \mathbf{C}_{\delta_{\node{v}_n}} f(\node{v}_1) \\
\mathbf{C}_{\delta_{\node{v}_1}} f(\node{v}_2) & \mathbf{C}_{\delta_{\node{v}_2}} f(\node{v}_2) & \ldots & \mathbf{C}_{\delta_{\node{v}_n}} f(\node{v}_2) \\
\vdots & \vdots & \ddots & \vdots \\
\mathbf{C}_{\delta_{\node{v}_1}} f(\node{v}_n) & \mathbf{C}_{\delta_{\node{v}_2}} f(\node{v}_n) & \ldots & \mathbf{C}_{\delta_{\node{v}_n}} f(\node{v}_n)
\end{pmatrix}\]
is symmetric and positive definite (positive semi-definite, respectively). 
\end{definition}
\end{samepage}

The kernel $K_f$ linked to the matrix $\mathbf{K}_{f}$ and the GBF $f$ is given as
\[ K_f(\node{v}_i,\node{v}_j) := \mathbf{C}_{\delta_{\node{v}_j}} f(\node{v}_i).\]
In this way, we can reformulate RLS problems in terms of GBF's. In particular, the unique minimizer $y^*$ of the RLS functional \eqref{eq:RLSfunctional} based on the kernel $K_f$ has the representation
$y^*(\node{v}) = \sum_{i = 1}^N c_i \mathbf{C}_{\delta_{\node{w}_i}} f(\node{v})$
with the coefficients $c_i$ given by \eqref{eq:computationcoefficients}. We call the corresponding minimizer $y^*$ the GBF-RLS solution. \\  

The signals $\mathbf{C}_{\delta_{\node{w}_i}} f$ can be interpreted as generalized translates of the basis function $f$ on the graph $G$. In fact, if $G$ has a group structure and the spectrum $\hat{G}$ consists of properly scaled characters of $G$, then the basis functions $\mathbf{C}_{\delta_{\node{w}_i}} f$ are shifts of the signal $f$ by the group element $\node{w}_i$. An advantage of GBF's is the following simple characterization of the kernel $K_f$ and the respective native space in terms of the graph Fourier transform. The derivations can be found in \cite{erb2019b}.

\begin{theorem} \label{thm:Bochner}
A GBF $f$ is positive definite if and only if $\hat{f}_k > 0$ for all $k \in \{1, \ldots, n\}$. The Mercer decomposition of the corresponding p.d. kernel $K_f$ is given by
\[ K_f(\node{v},\node{w}) = \mathbf{C}_{\delta_{\node{w}}}f(\node{v}) = \sum_{k=1}^n \hat{f}_k \, u_k(\node{v}) \, u_k(\node{w}).\]
The inner product and the norm of the native space
$\mathcal{N}_{K_f}$ can further be written as
\[ \langle x , y \rangle_{K_f} = 
\sum_{k=1}^n \frac{\hat{x}_k \, \hat{y}_k}{\hat{f}_k} = \hat{y}^\intercal \mathbf{M}_{1/\hat{f}} \, \hat{x} \quad \text{and} \quad \| x \|_{K_f} = \sqrt{\sum_{k=1}^n \frac{\hat{x}_k^2}{\hat{f}_k}}. \]
The RLS functional \eqref{eq:RLSfunctional} can be reformulated as
\begin{equation} \label{eq:RLSfunctionalpd}
y^* = \underset{x \in \mathcal{N}_{K_f}}{\mathrm{argmin}} \left( \frac{1}{N} \sum_{i=1}^N |x(\node{w}_i)-y(\node{w}_i)|^2 + \gamma x^\intercal \mathbf{U} \mathbf{M}_{1/\hat{f}} \mathbf{U}^\intercal x \right).
\end{equation}
\end{theorem}

\subsection{Examples of positive definite GBF's on graphs} \label{sec:examples}
\begin{enumerate}
\item[(1)] (Polynomials of the graph Laplacian $\Ll$)
Simple p.d. GBF's can be obtained from the eigenvalue decomposition of $\Ll$. We assume that $p_r$ is a polynomial of degree $r$ satisfying $p_r(\lambda_k)>0$ for all eigenvalues $\lambda_k$ of $\mathbf{L}$. Then, the spectral decomposition of $\Ll$ provides the p.d. matrix \cite{SmolaKondor2003}
\[p_r(\mathbf{L}) = \sum_{k=1}^n p_r(\lambda_k) u_k u_k^{\intercal}.\]
The corresponding generating GBF $f_{p_r(\mathbf{L})}$ is defined in terms of its Fourier transform as
\[ \hat{f}_{p_r(\mathbf{L})} = (p_r(\lambda_1), \ldots, p_r(\lambda_n)). \]
These p.d. GBF's get relevant when the size $n$ of $G$ is large. Then, the columns of
$p_r(\mathbf{L})$ can be calculated efficiently with simple 
matrix-vector multiplications using (a possibly sparse) $\Ll$.  
\item[(2)] (Variational or polyharmonic splines) Variational splines are based on the
kernel matrix
$$ (\epsilon \mathbf{I}_n + \mathbf{L})^{-s} = \sum_{k=1}^n \frac{1}{(\epsilon + \lambda_k)^s} u_k u_k^{\intercal},$$
being p.d. for $\epsilon > \max \{0, - \lambda_1\}$ and $s > 0$. They are studied in \cite{Pesenson2009,Ward2018interpolating} as
interpolants $y^\circ$ that minimize the functional $\|(\epsilon \mathbf{I}_n + \mathbf{L})^{s/2} x \|$. 
Variational splines can be regarded as GBF interpolants based on the p.d. GBF $f_{(\epsilon \mathbf{I}_n + \mathbf{L})^{-s}}$ defined in the spectral domain as
\[ \hat{f}_{(\epsilon \mathbf{I}_n + \mathbf{L})^{-s}} = \ts \left(\frac{1}{(\epsilon + \lambda_1)^s}, \ldots, \frac{1}{(\epsilon + \lambda_n)^s}\right). \]

\item[(3)] (Diffusion kernels)
The diffusion kernel on a graph \cite{KondorLafferty2002} based on the Mercer decomposition
\[e^{ -t \mathbf{L}} = \sum_{k=1}^n e^{ -t \lambda_k} u_k u_k^{\intercal}\]
is p.d. for all $t \in \Rr$. The Fourier transform of the respective p.d. GBF is given as
\[\hat{f}_{e^{-t \mathbf{L}}} = (e^{-t \lambda_1}, \ldots, e^{-t \lambda_n}). \]
\end{enumerate}

\section{Cartesian products of graphs} \label{sec:cartesianproduct}

\begin{definition}  \label{def:cartesianlaplacian}
Given two graphs $G = (V^G,E^G,\Ll^G)$ and $F = (V^F,E^F,\Ll^F)$ with $n$ and $n'$ nodes, we define the Cartesian product $G \times F$ as the graph with the node set 
$$V^{G \times F} = V^G \times V^F = \{(\node{v},\node{v'}) \ | \ \node{v} \in V^G, \ \node{v}' \in V^F\},$$
the edges
$$E^{G \times F} = \{ (e,e') \in V^{G \times F} \times V^{G \times F} \ | \ e \in E^G \ \text{or} \ e' \in E^F\},$$
and the graph Laplacian
\begin{equation} \label{eq:cartesianLaplacian}
\Ll^{G \times F} = \Ll^G \oplus \Ll^F = \Ll^G \otimes \mathbf{I}_{n'} + 
\mathbf{I}_{n} \otimes \Ll^F \in \Rr^{(n n') \times (n n')}
\end{equation}
corresponding to the Kronecker sum of the matrices $\Ll^G$ and $\Ll^F$. In this definition, $\otimes$ denotes the Kronecker (or tensor) product of two matrices. 
\end{definition}

\begin{remark}
Note that the classical unweighted adjacency matrix $\mathbf{A}^{G \times F}$ for the graph with the nodes $V^{G \times F}$ and the edges $E^{G \times F}$ is given by 
(see \cite[Theorem 33.5]{Hammack2011})
\begin{equation*}
\mathbf{A}^{G \times F} = \mathbf{A}^G \otimes \mathbf{I}_{n'} + 
\mathbf{I}_{n} \otimes \mathbf{A}^F.
\end{equation*}
The standard graph Laplacian $\mathbf{L}_S^{G \times F}$ for the pair 
$(V^{G \times F}, E^{G \times F})$ thus corresponds to the 
Kronecker product of the standard Laplacians $\mathbf{L}_S^{G}$ and $\mathbf{L}_S^{F}$ and, hence, to the definition given in equation \eqref{eq:cartesianLaplacian}. This is the motivation for the definition in equation \eqref{eq:cartesianLaplacian}. But, for other definitions of the graph Laplacian, as for instance the normalized Laplacian, this correspondence is in general not true. Anyway, the identity \eqref{eq:cartesianLaplacian} provides the relevant definition for us as it easily allows to combine arbitrary Laplacians when building the Cartesian product. 
\end{remark}

\begin{lemma} \label{lem:cartesiangraphlaplacian} The graph Laplacian $\Ll^{G \times F}$ has the following properties:
\begin{itemize}
\item[{(i)}] The matrix $\Ll^{G \times F} \in \Rr^{n n' \times n n'}$ is symmetric and a generalized graph Laplacian in the sense of \eqref{eq:generalizedLaplacian}.
\item[{(ii)}] If $\Ll^{G}$ and $\Ll^{F}$ are positive definite, then also $\Ll^{G \times F}$ is positive definite.
\item[{(iii)}] If the spectra of $\Ll^{G}$, $\Ll^{F}$ are contained in $[0,1]$, then the spectrum of $\Ll^{G \times F}$ is contained in $[0,2]$.
\end{itemize}
\end{lemma}

While $(i)$ can be derived directly from the defining identity in \eqref{eq:cartesianLaplacian}, the properties $(ii)$ and $(iii)$ of Lemma \ref{lem:cartesiangraphlaplacian} follow from the next well-known result related
to the eigendecomposition of $\Ll^{G \times F}$. 

\begin{lemma}[\cite{Hammack2011}, Proposition 33.6]
Let $\hat{G} = \{u_1^{G}, \ldots u_{n}^{G}\}$ and $\hat{F} = \{u_1^{F}, \ldots u_{n'}^{F}\}$ be the spectra of the graphs $G$ and $F$. Then 
\[\widehat{G \times F} = \{u_k^{G} \otimes u_{k'}^{F} \ | \ k \in \{1,\ldots, n\}, k' \in \{1, \ldots, n'\}\}\]
is a complete orthonormal system of eigenvectors of the Laplacian $\Ll^{G \times F}$ that we adopt as spectrum of the graph $G \times F$. The eigenvalue corresponding to the eigenvector $u_k^{G} \otimes u_{k'}^{F}$ is given by
$\lambda_k^{G} + \lambda_{k'}^{F}$. 
\end{lemma}

\begin{proof}
As the demonstration is very elementary, we include it at this place. $\hat{G}$ and $\hat{F}$ form complete orthonormal systems for $\mathcal{L}(G)$ and $\mathcal{L}(F)$, respectively. Therefore, $\widehat{G \times F}$ is an orthonormal basis for the space $\mathcal{L}(G \times F)$. Moreover, the definition \eqref{eq:cartesianLaplacian} of the graph Laplacian $\Ll^{G \times F}$ provides the identities
$$\Ll^{G \times F} (u_k^{G} \otimes u_{k'}^{F}) = \Ll^{G} u_k^{G} \otimes u_{k'}^{F} + 
 u_k^{G} \otimes \Ll^{F} u_{k'}^{F} = (\lambda_k^{G} + \lambda_{k'}^{F}) (u_k^{G} \otimes u_{k'}^{F}).$$ 
\end{proof}

For a signal $x \in \mathcal{L}(G \times F)$ we can naturally parametrize the graph Fourier transform $\hat{x}$ by a tuple $(k,k') \in \{1, \ldots, n\} \times \{1, \ldots, n'\}$ of indices and write
\begin{equation} \label{eq:GFT2D} \hat{x}_{k,k'} := (u_k^{G} \otimes u_{k'}^{F})^{\intercal} x = 
\sum_{i = 1}^{n} \sum_{i' = 1}^{n'} x(\node{v}_i , \node{v}_{i'}) u_{k}^{G}(\node{v}_i) u_{k'}^{F}(\node{v}_{i'}).
\end{equation}
With the Fourier matrix $\mathbf{U}^{G \times F}$ given as
\[\mathbf{U}^{G \times F} = \mathbf{U}^{G} \otimes \mathbf{U}^{F}.\]
we can also formulate the Fourier transform and its inverse more compactly as
\[ \hat{x} = (\mathbf{U}^{G \times F})^\intercal x, \quad \hat{x} = \mathbf{U}^{G \times F} x.\]
The convolution on $G \times F$ can then be written as
\[ x \ast y = \mathbf{U}^{G \times F} \mathbf{M}_{\hat{x}}(\mathbf{U}^{G \times F})^\intercal y. \]

\begin{lemma} \label{lem:tensorproduct} Let $f,f' \in \mathcal{L}(G)$ and 
$e,e' \in \mathcal{L}(F)$ be signals on their respective graphs. Then:
\begin{itemize}
\item[(i)] $(\widehat{f \otimes e})_{k,k'} = (\hat{f} \otimes \hat{e})_{k,k'}
= \hat{f}_k \, \hat{e}_{k'}$.
\item[(ii)] $(f \otimes e) \ast (f' \otimes e') = (f \ast f') \otimes (e \ast e')$.
\item[(iii)] If $f$ is p.d. on $G$ and $e$ is p.d. on 
$F$, then $f \otimes e$ is p.d. on $G \times F$.
\end{itemize}
\end{lemma}

\begin{proof}
$(i)$ With the Fourier transform on the Cartesian product $G \times F$ given in \eqref{eq:GFT2D}, the Fourier transform of the tensor product $f \otimes e$ reads as
\begin{equation*} (\widehat{f \otimes e})_{k,k'}  = 
\sum_{i = 1}^{n} \sum_{i' = 1}^{m} f(\node{v}_i)  e(\node{v}_{i'}) u_{k}^{G}(\node{v}_i) u_{k'}^{F}(\node{v}_{i'}) = \hat{f}_k \, \hat{e}_{k'}.\end{equation*}
$(ii)$ The definition of the convolution $\ast$ on $G \times F$ and the Fourier formula in $(i)$ yield
\[\widehat{((f \otimes e) \ast (f' \otimes e'))}_{k,k'} = 
\widehat{(f \otimes e)}_{k,k'} \widehat{(f' \otimes e')}_{k,k'}
= \hat{f}_k \, \hat{f'}_{k} \hat{e}_{k'} \, \hat{e'}_{k'}\
= \widehat{(f \ast f')}_{k} \widehat{(e \ast e')}_{k'}. \] 
Taking the inverse Fourier transform on both side (and applying $(i)$ on the right hand side), we get the stated result.\\
$(iii)$ The characterization of p.d. functions in Theorem \ref{thm:Bochner} in terms of their Fourier transform yields entirely positive Fourier coefficients $\hat{f}_k > 0$ and $\hat{e}_{k'}$ for the functions $f$ and $e$. Thus, by $(i)$ also all the Fourier coefficients $(\widehat{f \otimes e})_{k,k'} > 0$ of $f \otimes e$ are positive. This, on the other hand implies that $f \otimes e$ is p.d. on $G \times F$.
\end{proof}

\section{Feature-augmented GBF-RLS classification on graphs} \label{sec:psiGBFRLS}

\subsection{Construction of Feature-Augmented GBF's for semi-supervised classification}
\label{sec:construction}

\textbf{Goal.} Starting point for our classification task is a set of $N<n$ labels $y(\node{w}_1), \ldots, y(\node{w}_N) \in \{-1,1\}$ at the nodes
$W = \{\node{w}_1, \ldots, \node{w}_n\}$ of the graph $G$.
From this training set we aim at learning a function $y \in \mathcal{L}(G)$ that provides a classification $y(\node{v}) \in \{-1,1\}$ for all nodes $\node{v} \in V$ of the graph $G$. For this task, we will use a kernel-based RLS scheme as introduced in Section \ref{sec:kernelmethods}. 

In the following, we provide the general construction principle on how an ordinary p.d. GBF-kernel $K_f$ on $G$ can be augmented with additional feature information in order to obtain an augmented kernel
$K_{\vect{\psi}}$. For this, we will use the
tensor-product techniques introduced in the last section. \vspace{1mm}

\noindent \textbf{General setting.} In addition to the graph $G$, we assume to have $d \in \Nn$ feature graphs $F_1, \ldots, F_d$ and
$d$ corresponding feature maps $\psi_1: V \to V^{F_1}, \ldots, \psi_d: V \to V^{F_d}$. Further, for the graph $G$ and every feature graph $F_i$ we fix positive definite GBF's $f$ and $f^{F_i}$, $i \in \{1,\ldots,d\}$, respectively. The corresponding p.d. kernels are denoted by $K_f$ and $K_{f^{F_i}}$, respectively. 

\noindent \textbf{Construction.}
On the Cartesian product graph $G \times F_1 \times \cdots \times F_d$, we construct the tensor-product kernel $K_f \otimes K_{f_1} \otimes \cdots \otimes K_{f_d}$. Based on the embedding 
$$\vect{\psi}: V \to V \times V^{F_1} \times \cdots \times V^{F_d}, \quad \vect{\psi}(\node{v})  :=  (\node{v},\psi_1(\node{v}), \ldots, \psi_d(\node{v})),$$
we then define the feature-augmented kernel $K_{\vect{\psi}}$ on $V \times V$ as
\begin{align} \notag K_{\vect{\psi}}(\node{v},\node{w}) &:= \left(
\mathbf{C}_{\delta_{\node{w}}} f \otimes \mathbf{C}_{\delta_{\psi_1(\node{w})}} f^{F_1}
\otimes \cdots \otimes \mathbf{C}_{\delta_{\psi_d(\node{w})}} f^{F_d} \right)
(\vect{\psi}(\node{v})) \\ 
&= \mathbf{C}_{\delta_{\node{w}}} f(\node{v}) \, \mathbf{C}_{\delta_{\psi_1(\node{w})}} f^{F_1} (\psi_1(\node{v})) \cdots \mathbf{C}_{\delta_{\psi_d(\node{w})}} f^{F_d} (\psi_d(\node{v})) \notag \\
&= \underbrace{K_{f}(\node{v},\node{w})}_{\text{GBF centered at $\node{w}$}} 
\underbrace{K_{f^{F_1}}(\psi_1(\node{v}),\psi_1(\node{w})) \cdots K_{f^{F_d}}(\psi_d(\node{v}),\psi_d(\node{w}))}_{\text{Update by features $\psi_1, \ldots, \psi_d$}}. \label{eq:augmentedkernel}
\end{align}

The so-defined feature-augmented kernel $K_{\vect{\psi}}$ is not necessarily generated by a GBF. Nevertheless, as the columns $K_{\vect{\psi}}(\cdot, \node{w})$ can be regarded as updates of a GBF centered at $\node{w}$, we will refer to them as feature-augmented GBF's, or shortly, \psiGBF{}'s. 

The tensor-product construction implies the positive definiteness of the augmented kernel $K_{\vect{\psi}}$.

\begin{theorem} \label{thm-augmentedGBFRLS}
If $f$ is positive definite on $G$ and all $f^{F_i}$ are positive definite on $F_i$, $i\in\{1,\ldots,d\}$, then 
the augmented kernel $K_{\vect{\psi}}$ is positive definite on $G$. In particular, the RLS functional \eqref{eq:RLSfunctional} based on the augmented kernel $K_{\vect{\psi}}$ has a unique minimizer $y_{\vect{\psi}}^*$ with the representation  
\begin{equation} \label{eq:psirepresentertheorem}
y_{\vect{\psi}}^*(\node{v}) = \sum_{i = 1}^N c_i K_{\vect{\psi}}(\node{v},\node{w}_i).
\end{equation}
\end{theorem}

\begin{proof}
If $f$ and all $f^{F_i}$ are p.d., then Lemma \ref{lem:tensorproduct} $(iii)$ guarantees 
that the function $f \otimes f^{F_1} \otimes \cdots \otimes f^{F_d}$ is p.d. on the
Cartesian graph product $G \times F_1 \times \cdots \times F_d$. Consequently, Lemma \ref{lem:tensorproduct} $(ii)$ ensures that the kernel
\[K_f \otimes K_{f^{F_1}} \otimes \cdots \otimes K_{f^{F_d}} = 
K_{f \otimes f^{F_1} \otimes \cdots \otimes f^{F_d}}\]
is positive definite on the Cartesian product. By the inclusion principle (see \cite[Theorem 4.3.15]{HornJohnson1985}), this implies that also the principal subkernel given by
\[K_{\vect{\psi}}(\node{v},\node{w}) = 
(K_f \otimes K_{f^{F_1}} \otimes \cdots \otimes K_{f^{F_d}})(\vect{\psi}(\node{v}),\vect{\psi}(\node{w})), \quad \node{v},\node{w} \in V,\]
is positive definite. The representer theorem (see \cite{Schoelkopf2002}) then guarantees that the unique minimizer $y_{\vect{\psi}}^*$ of the RLS functional \eqref{eq:RLSfunctional} can be written in the form \eqref{eq:psirepresentertheorem}.
\end{proof}

\noindent \textbf{SSL classification.} Once the solution $y_{\vect{\psi}}^*(\node{v})$ of the augmented RLS problem in Theorem \ref{thm-augmentedGBFRLS} is calculated a binary classifier on the graph $G$ is obtained by calculating $\sgn (y_{\vect{\psi}}^*)$, where
\[ \sgn(y) = \begin{cases}-1&\text{if } y<0,\\ 0&\text{if } y=0, \\ 1& \text{if } y\geq 0.\end{cases}\]  

\begin{remark} Note that the presented augmentation strategy works also if the single kernels on the graphs $G$, $F_1, \ldots, F_d$ are not generated by GBF's. The usage of GBF's however simplifies the description of the kernels and allows to include geometric properties of the graphs into the kernels. Note also that instead of the presented tensor-product construction it is theoretically possible to define augmented kernels based on general kernels on the Cartesian product $G \times F_1 \times \cdots \times F_d$. In this case, it gets however more difficult to obtain relations between initial and augmented kernels. Also, the tensor-product construction allows to calculate the augmented kernel by a simple update procedure with a low computational complexity. We will investigate this further in Section \ref{sec:complexity}.
\end{remark}

\subsection{Algorithm to obtain the feature-augmented GBF-RLS classification}
The following Algorithm \ref{algorithm1} summarizes the procedure to compute the 
RLS solution in Theorem \ref{thm-augmentedGBFRLS} based on the feature-augmented kernel $K_{\vect{\psi}}$ defined in \eqref{eq:augmentedkernel}. As the construction is based on generating GBF's, we will denote the corresponding solution $y_{\vect{\psi}}^*$ as \psiGBFRLS{} solution.

\begin{algorithm}[H] \label{algorithm1}

\caption{RLS classification with feature-augmented graph basis functions (\psiGBF{}'s)}

\vspace{3mm}

\KwIn{$(i)$ $N$ labels $y(\node{w}_1), \ldots, y(\node{w}_N) \in \{-1,1\}$ at node set
$W = \{\node{w}_1, \ldots, \node{w}_n\}$ of the graph $G$. 
\newline $(ii)$ $d$ feature graphs $F_1, \ldots, F_d$, \newline
 $(iii)$ $d$ feature maps $\psi_1: V \to V^{F_1}, \ldots ,\psi_d: V \to V^{F_d}$,
 \newline
 $(iv)$ $d+1$ positive definite graph basis functions: $f:V \to \Rr$, and \newline
 \phantom{$(iv)$} $f^{F_1}:V^{F_1} \to \Rr, \ldots, f^{F_d}:V^{F_d} \to \Rr$.     
}

\vspace{1mm}

\textbf{Calculate} \\ \quad $(i)\phantom{i}$ $N$ generalized translates 
$\mathbf{C}_{\delta_{\node{w}_1}} f = \delta_{\node{w}_1} \ast f, \ldots, \mathbf{C}_{\delta_{\node{w}_N}} f = \delta_{\node{w}_N} \ast f$ on $G$, \\[1mm]
\quad $(ii)$ $N$ generalized translates  $\mathbf{C}_{\delta_{\psi_i(\node{w}_1})} f^{F_i} = \delta_{\psi_i(\node{w}_1}) \ast f^{F_i}, \ldots, \mathbf{C}_{\delta_{\psi_i(\node{w}_N})} f^{F_i} = \delta_{\psi_i(\node{w}_N}) \ast f^{F_i}$ \\ \quad \phantom{$(ii)$} on the $d$ feature graphs $F_i$, $i \in \{1,\ldots,d\}$.

\vspace{2mm}

\textbf{Construct} the feature-augmented graph basis functions (\psiGBF{}'s):
\begin{align*} K_{\vect{\psi}}(\node{v},\node{w}_k) &= \left(
\mathbf{C}_{\delta_{\node{w}_k}} f \otimes \mathbf{C}_{\delta_{\psi_1(\node{w}_k)}} f^{F_1}
\otimes \cdots \otimes \mathbf{C}_{\delta_{\psi_d(\node{w}_k)}} f^{F_d} \right)
(\node{v}, \psi_1(\node{v}), \ldots, \psi_d(\node{v})) \\ 
&= K_{f}(\node{v},\node{w}_k) \, K_{f^{F_1}}(\psi_1(\node{v}),\psi_1(\node{w}_k)) \cdots K_{f^{F_d}}(\psi_d(\node{v}),\psi_d(\node{w}_k)), \quad k \in \{1, \ldots, N\}.
\end{align*}

\vspace{2mm}

\textbf{Solve} the linear system of equations

\begin{equation*} \label{eq:computationcoefficientsGBF} 
\underbrace{\begin{pmatrix} K_{\vect{\psi}}(\node{w_1},\node{w}_1)  + \gamma N & K_{\vect{\psi}}(\node{w_2},\node{w}_1)  & \ldots & K_{\vect{\psi}}(\node{w_N},\node{w}_1)  \\
K_{\vect{\psi}}(\node{w_1},\node{w}_2)  & K_{\vect{\psi}}(\node{w_2},\node{w}_2)  + \gamma N & \ldots & K_{\vect{\psi}}(\node{w_N},\node{w}_2)  \\
\vdots & \vdots & \ddots & \vdots \\
K_{\vect{\psi}}(\node{w_1},\node{w}_N)  & K_{\vect{\psi}}(\node{w_2},\node{w}_N)  & \ldots & K_{\vect{\psi}}(\node{w_N},\node{w}_N)  + \gamma N
\end{pmatrix}}_{\mathbf{K}_{\vect{\psi},W} + \gamma N \mathbf{I}_N} 
\begin{pmatrix} c_1 \\ c_2 \\ \vdots \\ c_N \end{pmatrix}
= \begin{pmatrix} y(\node{w}_1) \\ y(\node{w}_2) \\ \vdots \\ y(\node{w}_N) \end{pmatrix}.
\end{equation*}

\textbf{Calculate} the feature-augmented \psiGBFRLS{} solution
\[ y_{\vect{\psi}}^*(\node{v}) = \sum_{k=1}^N c_k K_{\vect{\psi}}(\node{v},\node{w}_k).\]

\noindent A \textbf{binary classification} on $G$ according to this solution is then given by $\sgn (y_{\vect{\psi}}^*)$.
\end{algorithm}

\subsection{Computational complexities} \label{sec:complexity}

\noindent{\bfseries Complexity of the augmentation step.}
The augmentation formula in \eqref{eq:augmentedkernel} allows to construct the new feature-augmented kernel from an initial kernel by a simple update strategy. When adding a single extra feature it is not necessary to calculate an entirely new kernel on a Cartesian product, only a simple Schur-Hadamard product between the initial kernel and the feature information based on the just added feature kernel is required. This considerably lowers the computational costs as all operations can be performed on the nodes of the graph $G$ and not on the Cartesian product $G \times F_1 \times \cdots \times F_d$. With known feature kernel and feature map, a single update step can be performed in $n N$ arithmetic multiplications, for $d$ features this results in $d n N$ multiplications.  

\noindent{\bfseries Complexity of the computation of the RLS solution.}
As we suppose that the number of labels $N$ is low, the calculation of the RLS solution $y_{\vect{\psi}}^*$ itself is inexpensive. The solution of the linear system \eqref{eq:computationcoefficients} for the expansion coefficients can be performed in 
$\mathcal{O}(N^3)$ arithmetic operations, the summation of the basis functions in \eqref{eq:psirepresentertheorem} requires $(2N-1) n $ arithmetic operations. 

\noindent{\bfseries Complexity of the calculation of the GBF's.}
The calculation of this step requires some effort if it is necessary to calculate the entire eigendecomposition of the graph Laplacian $\Ll$. In this case, $\mathcal{O}(n^3)$ arithmetic steps are required to calculate the spectrum of the graph. This is certainly a drawback for larger graphs. In such a scenario, it is recommendable to use GBF's that can be calculated without the graph Fourier transform. This is for instance possible for the GBF's in Section \ref{sec:examples} (1). In this example, the entire set of shifted GBF's $\mathbf{C}_{\delta_{\node{w}_1}} f, \ldots, \mathbf{C}_{\delta_{\node{w}_n}} f$ can be calculated in $\mathcal{O}(n^2 N)$ arithmetic operations. A further reduction is possible if the graph Laplacian $\Ll$ is a sparse matrix.  

\section{Examples of feature maps and feature graphs} \label{sec:featuremaps}

\subsection{Binary classifications as features} \label{sec:examplebinaryclassification}
A priori known or learned classifications of the graph nodes are valid additional features that allow to incorporate unlabeled nodes in a SSL scheme. In this first example, we focus on binary classifications and describe how they can be included in the augmentation step with help of explicit feature maps and feature graphs.  

Starting with a given binary classification $\psi_{\bin}: V \to \{-1,1\}$, we model the feature graph $F_{\bin}$ as a binary graph consisting of two nodes $V^{F_{\bin}} = \{-1,1\}$, a single undirected edge connecting these two nodes, and the (standard) graph Laplacian 
\begin{equation} \label{eq:binlaplacian} \Ll^{F_{\bin}} = \begin{pmatrix} 1 & -1 \\ -1 & 1 \end{pmatrix}.\end{equation}
The spectral decomposition of this graph Laplacian is given as
\[ \Ll^{F_{\bin}} = \begin{pmatrix} \frac{1}{\sqrt2} & - \frac{1}{\sqrt2} \\ \frac{1}{\sqrt2} &  \frac{1}{\sqrt2} \end{pmatrix} 
\begin{pmatrix} 0 & 0 \\ 0 & 2 \end{pmatrix} \begin{pmatrix} \frac{1}{\sqrt2} & -\frac{1}{\sqrt2} \\ \frac{1}{\sqrt2} &  \frac{1}{\sqrt2} \end{pmatrix}^{\intercal}.\]
We define a correlation kernel $K^{F_{\bin}}$ on the binary graph $F_{\bin}$ in terms of the matrix
\begin{equation} \label{eq:bincorrelationmatrix}
\Kk^{F_{\bin}} = \begin{pmatrix} 1 & \alpha \\ \alpha & 1 \end{pmatrix}
= (1 + \alpha)  \mathbf{I}_2 - \alpha \Ll^{F_{\bin}},
\end{equation}
which is positive semi-definite in the parameter range $-1 \leq \alpha \leq 1$. The matrix $\Kk^{F_{\bin}}$ is a (linear) polynomial of the matrix $\Ll^{F_{\bin}}$, and generated by the GBF $f_{\bin} = \sqrt{2} (\alpha, 1)^{\intercal}$ with the 
Fourier transform $\hat{f}_{\bin} = (1+\alpha,1-\alpha)^{\intercal}$. Finally, the update matrix in \eqref{eq:augmentedkernel} for a single binary classification is explicitly given by
$\Kk^{F_{\bin}}(\psi_{\bin}(\node{v}),\psi_{\bin}(\node{w}))$.  

\subsection{Similarity graphs to incorporate attributes} \label{sec:similarity}

Attributes on the nodes $\node{v}$ of the graph $G$ are usually encoded 
as vectors $r_{\node{v}}$ in $\Rr^d$. The set of all attributes forms a point cloud $V^{F_{\SIM}} \subset \Rr^d$ of $n' \leq n$ elements. Generating a similarity graph $F_{\SIM}$ out of this point cloud, we gather additional information of the unlabeled nodes that can be integrated in a SSL scheme. A simple way to define the similarity graph $F_{\SIM}$ from the point cloud $V^{F_{\SIM}}$ is to consider the adjacency matrix $\mathbf{A}^{F_{\SIM}}$ with the edge weights
\[\mathbf{A}^{F_{\SIM}}_{i,j} = \mathrm{e}^{- \alpha \|r_{\node{v}_i} - r_{\node{v}_j}\|^2}, \quad i,j \in \{1, \ldots, n'\}. \]
A corresponding graph Laplacian can then be defined as $\mathbf{L}^{F_{\SIM}} = -\mathbf{A}^{F_{\SIM}}$. As a correlation kernel $K^{F_{\SIM}}$ on $F_{\SIM}$ we can further use the positive definite adjacency matrix $\mathbf{A}^{F_{\SIM}}$, i.e., we set
\begin{equation} \label{eq:simcorrelationmatrix} 
\Kk^{F_{\SIM}} = \mathbf{A}^{F_{\SIM}}.
\end{equation}
The kernel matrix $\Kk^{F_{\SIM}}$ describes the correlation between the attributes $r_{\node{v}}$ and is generated by a p.d. GBF $f_{\SIM}$. The GBF $f_{\SIM}$ can be characterized in terms of the graph Fourier transform on $F_{\SIM}$ as
\[\hat{f}_{\SIM} = (-\lambda_{n'}, - \lambda_{n'-1}, \ldots - \lambda_1)^{\intercal},\]
where $\lambda_1, \ldots, \lambda_{n'}$ are the increasingly ordered (negative) eigenvalues
of the Laplacian $\mathbf{L}^{F_{\SIM}}$. The feature map $\psi_{\SIM}$ on $G$ mapping on the similarity graph $F_{\SIM}$ is in this case given by
$\psi_{\SIM}(\node{v}) = r_{\node{v}}$. 

\subsection{Feature maps via unsupervised learning}

Unsupervised learning methods provide automated ways to obtain feature maps on graphs. Especially graph clustering algorithms \cite{schaeffer2007} are powerful methods to extract geometric information of the graph. This information can then be used to generate feature graphs and feature maps as for instance described in the last two subsections. Prominent examples of such clustering methods are gaussian mixture models, density-based spatial clustering, spectral clustering or $k$-means. In the numerical examples at the end of this section we will apply a spectral clustering based on the normalized-cut method of Shi and Malik \cite{shimalik2000} to obtain an unsupervised binary classification of the graph nodes. 

\section{Augmentation with binary classification feature} \label{sec:augmentationbinary}

We want to get a better understanding on how feature augmentation affects a given kernel. To have a framework in which theoretical statements are possible, we will focus on a simpler kernel setting in which a given GBF-kernel is augmented solely with a binary classification feature. 

\begin{assumption} \label{ass:1}
Throughout this section, we assume that:
\begin{itemize}
\item[(i)] $f$ is a positive definite GBF on $G$ with corresponding kernel $K_f$ and kernel matrix $\Kk_f$. 
\item[(ii)] $\psi_{\bin}: V \to V^{F_{\bin}}$ is a binary feature map on $G$. 
\item[(iii)] The feature kernel $K^{F_{\bin}}$ is given by the graph Laplacian \eqref{eq:binlaplacian} on $F_{\bin}$, i.e., $\Kk^{F_{\bin}}=\Ll^{F_{\bin}}$.
\item[(iv)] The augmented kernel $K_{\vect{\psi}}$ is given by
\begin{equation*} \label{eq:binaryaugmentedkernel} K_{\vect{\psi}}(\node{v}, \node{w}) = K_f(\node{v},\node{w}) K^{F_{\bin}} (\psi_{\bin}(\node{v}), \psi_{\bin}(\node{w})).
\end{equation*}
\end{itemize}
\end{assumption}
Note that the given feature matrix $\Kk^{F_{\bin}} = \Ll^{F_{\bin}}$ is only positive semi-definite. This is however not a restriction for our upcoming considerations as the following Lemma demonstrates.  

\begin{lemma} \label{lem:binaugmentedkernel} In the setting of Assumption \ref{ass:1}, the augmented kernel $K_{\vect{\psi}}$ has the following properties:
\begin{itemize}
\item[(i)] $K_{\vect{\psi}}(\node{v},\node{w}) = \psi_{\bin}(\node{v})
\mathbf{C}_{\delta_{\node{w}}} f (\node{v}) \psi_{\bin}(\node{w})$.
\item[(ii)] The Mercer decomposition of $K_{\vect{\psi}}$ is given by
\[K_{\vect{\psi}}(\node{v},\node{w}) = \sum_{k = 1}^n \hat{f}_k 
(\psi_{\bin}(\node{v}) u_k(\node{v})) (\psi_{\bin}(\node{w}) u_k(\node{w})). \]
Moreover, $\{ \psi_{\bin} u_1, \ldots, \psi_{\bin} u_n\}$ forms a complete orthonormal system of eigenvectors of the matrix $\Kk_{\vect{\psi}}$ with corresponding eigenvalues
$\hat{f}_1, \ldots, \hat{f}_n$. In particular, the kernel $K_{\vect{\psi}}$ is positive definite. 
\item[(iii)] For all subsets $W \subseteq V$ the matrices $\Kk_{f,W}$ and 
$\Kk_{\vect{\psi},W}$ are similar and have the same eigenvalues with the same geometric multiplicities.  
\end{itemize}
\end{lemma}

\begin{proof}
$(i)$ If $\Kk^{F_{\bin}}$ corresponds to the Laplacian $\Ll^{F_{\bin}}$ given in \eqref{eq:binlaplacian}, we have $K^{F_{\bin}} (\psi_{\bin}(\node{v}), \psi_{\bin}(\node{w})) = \psi_{\bin}(\node{v}) \psi_{\bin}(\node{w})$. In the setting of Assumption \ref{ass:1}, this implies $(i)$.\\
$(ii)$ By Theorem \ref{thm:Bochner}, the Mercer decomposition of $K_f$ is given by 
$\mathbf{C}_{\delta_{\node{w}}}f(\node{v}) = \sum_{k=1}^n \hat{f}_k \, u_k(\node{v}) \, u_k(\node{w})$. Plugging this into the identity $(i)$, we directly obtain the Mercer decomposition in $(ii)$. Herein, the fact that $(\psi_{\bin} u_k)^\intercal (\psi_{\bin} u_{k'}) = u_k^\intercal u_{k'} = \delta_{k,k'}$ ensures that the system $\{\psi_{\bin} u_k \ | \ k \in \{1,\ldots,n\}\}$ is a complete orthonormal system of eigenvectors of $\Kk_f$.\\
$(iii)$ This is as well an immediate consequence of the identity $(i)$. 
\end{proof}

\subsection{Consistency results}

\begin{theorem} \label{thm:consistency}
In addition to Assumption \ref{ass:1}, suppose that $K_f$ is positive on $G$, i.e., $K_f(\node{v},\node{w}) > 0$ for all $\node{v},\node{w} \in V$. Then, we get the following consistency results for the \psiGBFRLS{} solution $y_{\vect{\psi}}^*$:
\begin{itemize}
\item[(i)] If $N = 1$ and the single label $y(\node{w}_1) \in \{-1,1\}$ is correctly classified by $\psi_{\bin}$, i.e., $\psi_{\bin}(\node{w}_1) = y(\node{w}_1)$, then  
\[\sgn( y_{\vect{\psi}}^*(\node{v})) = \sgn( y(\node{w}_1) K_{\psi} (\node{v}, \node{w}_1) ) = \sgn ( \psi_{\bin}(\node{v}) \mathbf{C}_{\delta_{\node{w}_1}} f (\node{v}) ) = \psi_{\bin}(\node{v}).\]
\item[(ii)] If $N > 1$ labels $y(\node{w}_1), \ldots, y(\node{w}_N) \in \{-1,1\}$ are correctly classified by the binary classificator $\psi_{\bin}$ and the regularization parameter $\gamma$ for the RLS scheme satisfies $\gamma > 
2 \max_{\node{v}} K_{f}(\node{v},\node{v})$, then  
\[\sgn( y_{\vect{\psi}}^*(\node{v}) ) = \psi_{\bin}(\node{v}).\]
\end{itemize}
\end{theorem}

\begin{proof}
$(i)$ For one single label $y(\node{w}_1)$, the \psiGBFRLS solution is given by
\[ y_{\vect{\psi}}^*(\node{v}) = \frac{y(\node{w}_1)}{K_{\vect{\psi}}(\node{w}_1,\node{w}_1) + \gamma}K_{\vect{\psi}}(\node{v},\node{w}_1).\] 
With the identity in Lemma \ref{lem:binaugmentedkernel} $(i)$ and the fact that 
$K_{\vect{\psi}}(\node{w}_1,\node{w}_1) > 0$, we then obtain 
\[\sgn( y_{\vect{\psi}}^*(\node{v})) = \sgn( y(\node{w}_1) K_{\psi} (\node{v}, \node{w}_1) ) = \sgn ( y(\node{w}_1) \psi_{\bin}(\node{w}_1) \mathbf{C}_{\delta_{\node{w}_1}} f (\node{v}) \psi_{\bin}(\node{v})) = \sgn ( \mathbf{C}_{\delta_{\node{w}_1}} f (\node{v}) \psi_{\bin}(\node{v})).\]
As the kernel $K_f$ is positive on $G$, we can thus conclude that
\[\sgn( y_{\vect{\psi}}^*(\node{v})) = \psi_{\bin}(\node{v}).\]
$(ii)$ By Lemma \ref{lem:binaugmentedkernel} $(i)$, the $1$-norm of the matrix $\Kk_{\vect{\psi},W}$ is identical to the $1$-norm of $\Kk_{f,W}$ which is bounded by 
$N \max_{\node{v}} K_{f}(\node{v},\node{v})$. Therefore, for $\gamma > 
2 \max_{\node{v}} K_{f}(\node{v},\node{v})$ the matrix $\Kk_{\vect{\psi},W} + \gamma N \mathbf{I}_n$ is strictly diagonally dominant and the $1$-norm of $\Kk_{\vect{\psi},W}/(\gamma N)$ is less than $1/2$. By applying the Neumann series expansion, we have
\[(\Kk_{\vect{\psi},W} + \gamma N \mathbf{I}_n)^{-1} = \frac{1}{\gamma N}\left( \mathbf{I}_n + \sum_{k=1}^\infty( \ts \frac{1}{\gamma N} \Kk_{\vect{\psi},W})^k \right).\] As the $1$-norm of $\Kk_{\vect{\psi},W}/(\gamma N)$ is less than $1/2$, the sum on the right hand gives a matrix with $1$-norm less than $1$. This implies that 
the matrix $(\Kk_{\vect{\psi},W} + \gamma N \mathbf{I}_n)^{-1}$ is strictly diagonally dominant with positive diagonal entries. This on the other hand implies that 
\[\sgn \left( (\Kk_{\vect{\psi},W} + \gamma N \mathbf{I}_n)^{-1}
(y(\node{w}_1), \ldots, y(\node{w}_N))^{\intercal} \right) = (y(\node{w}_1), \ldots, y(\node{w}_N))^{\intercal}. \]
For the solution of the RLS problem, we then obtain
$$\sgn (y_{\vect{\psi}}^*(\node{v})) = 
\psi_{\bin}(\node{v}) \, \sgn \left( (K_f(\node{v},\node{w}_1)
\psi_{\bin}(\node{w}_1), \ldots, K_f(\node{v},\node{w}_N)\psi_{\bin}(\node{w}_N) )(\Kk_{\vect{\psi},W} + \gamma N \mathbf{I}_n)^{-1} \begin{pmatrix}
y(\node{w}_1) \\ \vdots \\ y(\node{w}_N) \end{pmatrix} \right).$$
As the labels $y(\node{w}_i)$ correspond to the prior values $\psi_{\bin}(\node{w}_i)$
and the kernel $K_f$ is assumed to be strictly positive, the sign of the second term on the right hand side is always positive, and therefore, $\sgn (y_{\vect{\psi}}^*(\node{v})) = 
\psi_{\bin}(\node{v})$ for all $\node{v} \in V$. 
\end{proof}

Theorem \ref{thm:consistency} contains two important messages. 
If only one labeled node is given, then, by $(i)$, the given SSL schemes reproduces the prior
$\psi_{\bin}$. Equally important, if the regularization parameter $\gamma$ is chosen large enough, the resulting classification $\sgn( y_{\vect{\psi}}^*)$ corresponds also to the prior. This means, that in the reproducing kernel Hilbert space $\mathcal{N}_{K_{\vect{\psi}}}$ of the augmented kernel the prior $\psi_{\bin}$ itself can be regarded as a smooth function with a small Hilbert space norm. 

As prerequisite in Theorem \ref{thm:consistency} we assume that the entries of the kernel $K_f$ are all positive. The following proposition provides us with two important families of such positive p.d. kernels.
\begin{proposition}
Let $\Ll \in \Rr^{n \times n}$ be a generalized Laplacian on $G$ and $\mathrm{d} = \underset{i \in \{1, \ldots, n\}}{\max} \Ll_{i,i}$.
\begin{itemize}
\item[(i)] The diffusion kernel on $G$ given by the matrix exponential $e^{-t \Ll}$, $t \geq 0$, is non-negative. If $G$ is a connected graph, then all entries of $e^{-t \Ll}$, $t > 0$, are strictly positive.
\item[(ii)] If $\|\mathrm{d} \mathbf{I}_n - \Ll\| < \mathrm{d} + \epsilon$ then the variational spline kernel given by $(\epsilon \mathbf{I}_n + \mathbf{L})^{-s}$ is positive definite and non-negative for $s \in \Nn$ and $\epsilon > - \mathrm{d}$. If $G = (V,E,\Ll)$ is connected, then all entries of $(\epsilon \mathbf{I}_n + \mathbf{L})^{-s}$ are strictly positive.
\end{itemize}
\end{proposition}

\begin{proof}
$(i)$ We can write the Laplacian as $\Ll = \Dd - \Aa$ with the diagonal matrix $\mathbf{D} = \mathrm{d} \mathbf{I}_n$, and the remainder $\Aa = \mathbf{D} - \Ll$. As $\mathrm{d} = \max_i \Ll_{i,i}$ and $\Ll$ is defined as a generalized Laplacian, the entries of $\Aa$ are all non-negative. For this, the matrix $-\mathbf{L} = \Aa - \Dd $ is essentially non-negative and \cite[Chap. 6, Theorem (3.12)]{Berman1994} implies that the exponential $e^{-t \Ll}$ is non-negative for $t \geq 0$ (in fact, the proof of this statement is quite elementary). Further, the graph $G$ is connected if and only if the generalized Laplacian $\Ll$ is an irreducible matrix. Again, by \cite[Chap. 6, Theorem (3.12)]{Berman1994}, this implies that all entries of $e^{-t \Ll}$ are strictly positive for $t > 0$. 

$(ii)$ If $\|\Dd - \Ll\| < \mathrm{d} + \epsilon$, a Neumann series expansion can be applied in order the represent $(\epsilon \mathbf{I}_n + \mathbf{L})^{-s}$:
\[(\epsilon \mathbf{I}_n + \mathbf{L})^{-s} = ( (\epsilon + \mathrm{d}) \mathbf{I}_n - \mathbf{A})^{-s} = \left( \frac{1}{\epsilon + \mathrm{d}} \sum_{k=0}^\infty( \ts \frac{1}{\epsilon + \mathrm{d}} \Aa)^k \right)^s.\]
Therefore, as all entries of $\Aa$ are non-negative, $\epsilon + \mathrm{d} > 0$ and $s \in \Nn$, also all entries of $(\epsilon \mathbf{I}_n + \mathbf{L})^{-s}$ are non-negative. Further, if $G$ is connected, then $\Aa$ is irreducible, and all entries of 
$(\epsilon \mathbf{I}_n + \mathbf{L})^{-s}$ are strictly positive. 
\end{proof}

\subsection{Comparison of Error Estimates} \label{sec:errorbounds}

We want to go a step further in our analysis of the augmentation step and consider, for a given 
classification $y$ on $G$, the dependence of the RLS approximation error $|y(\node{v}) -  y_{\vect{\psi}}^*(\node{v})|$ in terms of the binary feature map $\psi_{\bin}$. For this, we split the approximation error in two parts 
\begin{equation*}
|y(\node{v}) -  y_{\vect{\psi}}^*(\node{v})| \leq 
\underbrace{|y(\node{v}) -  \itpx(\node{v})|}_{\text{Interpolation error}} + \underbrace{|\itpx(\node{v}) -  y_{\vect{\psi}}^*(\node{v})|}_{\text{Regularization error}}.
\end{equation*} 
As introduced in Section \ref{subsec:kernelRLS}, $\itpx$ denotes the uniquely determined interpolant for the data $(\node{w}_k,y(\node{w}_k))$, $k \in \{1,\ldots,N\}$, in the native space $\mathcal{N}_{W,K_{\vect{\psi}}}$ for a vanishing regularization parameter $\gamma$. In a reproducing kernel Hilbert space setup, the interpolation error can be estimated as (cf. \cite{erb2019b}) 
\[ |y(\node{v}) -  \itpx(\node{v})| \leq  P_{W,K_{\vect{\psi}}}(\node{v}) \|y\|_{K_{\vect{\psi}}},\]
with the power function given as 
\begin{equation} \label{eq:powerfunction} P_{W,K_{\vect{\psi}}}(\node{v})
= \left\| K_{\vect{\psi}}(\cdot,\node{v}) - \sum_{k=1}^N \ell_{\vect{\psi},k}(\node{v}) K_{\vect{\psi}}(\cdot,\node{w}_k) \right\|_{K_{\vect{\psi}}}.
\end{equation}
Here, the Lagrange basis functions $\ell_{\vect{\psi},k}$ are defined as the functions in $\mathcal{N}_{W,K_{\vect{\psi}}}$ that interpolate the canonical basis function $\delta_{\node{w}_k}$ at the nodes $W$. 
\begin{lemma} \label{lem-powerfunction}
The power function $P_{W,K_{\vect{\psi}}}(\node{v})$ is independent of the 
binary feature map $\psi_{\bin}$, that is $P_{W,K_{\vect{\psi}}}(\node{v}) = P_{W,K_{f}}(\node{v})$ for all possible feature maps $\psi_{\bin}: V \to V^{F_{\bin}}$. 
\end{lemma}

\begin{proof} It is well-known (see \cite[Theorem 11.1]{Scha98}, \cite[Theorem 11.5]{We05}) that $\sum_{k=1}^N \ell_{\vect{\psi},k}(\node{v}) K_{\vect{\psi}}(\cdot,\node{w}_k)$ is the best approximation of $K_{\vect{\psi}}(\cdot,\node{v})$ in the subspace $\mathcal{N}_{K_{\vect{\psi}},W}$ with respect to the native space norm in $\mathcal{N}_{K_{\vect{\psi}}}$. Therefore, we get an upper estimate of $P_{W,K_{\vect{\psi}}}(\node{v})$ by substituting the coefficients $\ell_{\vect{\psi},k}(\node{v})$ in \eqref{eq:powerfunction} by the coefficients $\psi_{\bin}(\node{v}) 
\psi_{\bin}(\node{w}_k)\ell_{f,k}(\node{v})$. Here, $\ell_{f,k}$ denotes the 
Lagrange basis function with respect to the node $\node{w}_k$ in the original native space $\mathcal{N}_{K_{f},W}$. Using this substitution and the identity in Lemma \ref{lem:binaugmentedkernel} $(i)$, we get
\begin{align*} P_{W,K_{\vect{\psi}}}(\node{v})
&\leq \left\| K_{\vect{\psi}}(\cdot,\node{v}) - \sum_{k=1}^N \psi_{\bin}(\node{v}) 
\psi_{\bin}(\node{w}_k)\ell_{f,k}(\node{v}) K_{\vect{\psi}}(\cdot,\node{w}_k) \right\|_{K_{\vect{\psi}}} \\ &= \left\| \psi_{\bin} \left( K_{f}(\cdot,\node{v}) - \sum_{k=1}^N \ell_{f,k}(\node{v}) K_{f}(\cdot,\node{w}_k)\right) \psi_{\bin}(\node{v}) \right\|_{K_{\vect{\psi}}} = P_{W,K_{f}}(\node{v}).
\end{align*}
With an identical argumentation, we also obtain the reverse inequality $P_{W,K_{f}}(\node{v}) \leq P_{W,K_{\vect{\psi}}}(\node{v})$, and, thus, the statement of the lemma. 
\end{proof}

We next turn our attention to the regularization error.

\begin{lemma} \label{lem:regularizationerror} The regularization error is bounded by
\[ |\itpx(\node{v}) -  y_{\vect{\psi}}^*(\node{v})| \leq  \frac{\gamma N\sqrt{K_f(\node{v},\node{v})}}{\lambda_{\min}(\Kk_{f,W})+\gamma N } \|y\|_{K_{\vect{\psi}}},\]
where $\lambda_{\min}(\Kk_{f,W})$ denotes the smallest eigenvalue of $\Kk_{f,W}$.
\end{lemma}

\begin{proof}
This proof consists mostly in the combination of standard estimates and identities for reproducing kernel Hilbert spaces. Compared to similar estimates as, for instance, given in \cite{Saitoh2005}, we have to consider also the dependence on the feature map $\psi_{\bin}$. 
Based on the representation \eqref{eq:psirepresentertheorem} for the RLS solution $y_{\vect{\psi}}^*$ and the interpolant $\itpx$ (with vanishing $\gamma = 0$), we first obtain the identity
\begin{align*}|\itpx(\node{v}) -  y_{\vect{\psi}}^*(\node{v})| &= \left| 
\left(K_{\vect{\psi}}(\node{v},\node{w}_1), \ldots, K_{\vect{\psi}}(\node{v},\node{w}_N)\right)\left( \Kk_{\vect{\psi},W}^{-1} - (\Kk_{\vect{\psi},W} + \gamma N \mathbf{I}_n)^{-1}\right) (
y(\node{w}_1), \ldots, y(\node{w}_N))^\intercal\right| \\
&= \gamma N \left| 
\left(K_{\vect{\psi}}(\node{v},\node{w}_1), \ldots, K_{\vect{\psi}}(\node{v},\node{w}_N)\right) \Kk_{\vect{\psi},W}^{-1}(\Kk_{\vect{\psi},W} + \gamma N \mathbf{I}_n)^{-1} (
y(\node{w}_1), \ldots, y(\node{w}_N))^\intercal \right|.
\end{align*}
Using the Cauchy-Schwarz inequality, we therefore get 
\[|\itpx(\node{v}) -  y_{\vect{\psi}}^*(\node{v})| \leq \frac{\gamma N \overbrace{\| \Kk_{\vect{\psi},W}^{-1/2}\left(K_{\vect{\psi}}(\node{v},\node{w}_1), \ldots, K_{\vect{\psi}}(\node{v},\node{w}_N)^{\intercal}\right)\|}^{\sqrt{K_f(\node{v},\node{v})}} \overbrace{\| \Kk_{\vect{\psi},W}^{-1/2}(
y(\node{w}_1), \ldots, y(\node{w}_N))^\intercal\|}^{\|\itpx\|_{K_{\vect{\psi}}}} } {\|(\Kk_{\vect{\psi},W} + \gamma N \mathbf{I}_n)^{-1}\|^{-1}}.
 \]
As the interpolant $\itpx$ satisfies $\|\itpx\|_{K_{\vect{\psi}}} \leq \|y\| _{K_{\vect{\psi}}}$ (cf. \cite[Corollary 10.25]{We05}), and $\lambda_{\min}(\Kk_{f,W})+\gamma N \leq \|(\Kk_{\vect{\psi},W} + \gamma N \mathbf{I}_n)^{-1}\|^{-1}$ (by Lemma \ref{lem:binaugmentedkernel}), we obtain the desired estimate. 
\end{proof}

The combination of the two Lemmas \ref{lem-powerfunction} and \ref{lem:regularizationerror} gives the following overall error bound. 

\begin{theorem} \label{thm-errorestimate}
In the setting of Assumption \ref{ass:1}, we get the error bound
\[ |y(\node{v}) - y_{\vect{\psi}}^* (\node{v})| \leq \left(P_{W,K_{f}}(\node{v}) +  \gamma N \frac{\sqrt{K_f(\node{v},\node{v})}}{\lambda_{\min}(\Kk_{f,W})+\gamma N} \right) \|y\|_{K_{\vect{\psi}}}.\]
\end{theorem}

\begin{remark} On the right hand side of the error estimate in Theorem \ref{thm-errorestimate}, the term in brackets is independent of the binary feature $\psi_{\bin}$ and the feature information is only contained in $\|y\|_{K_{\vect{\psi}}}$. This indicates that the approximation error $|y(\node{v}) - y_{\vect{\psi}}^* (\node{v})|$ will be small if the given data $y$ is smooth with respect to the native space norm $\|y\|_{K_{\vect{\psi}}}$ imposed by the feature $\psi_{\bin}$. According to this a feature choice $\psi_{\bin}$ is advantageous if it keeps the native space norm $\|y\|_{K_{\vect{\psi}}}$ small.

In \cite{erb2019b} it is shown that the power function $P_{W,K_{f}}(\node{v})$ in the error estimate is bounded by
\[ P_{W,K_{f}}(\node{v}) \leq (1+\normconst) 
\left( \sum_{k=M+1}^n \hat{f}_k \right)^{1/2}. \]
In particular, $P_{W,K_{f}}$ is small if the norming constant $\normconst$ is small and the GBF $f$ decays rapidly in the Fourier domain. The norming constant $\normconst$ itself describes how well the subset $W$ is able to resolve the space $\mathcal{B}_M = \mathrm{span} \{u_1, \ldots, u_M\}$ of bandlimited functions. In particular, $\normconst < \infty$ holds true if and only if the linear mapping $x \to (x(\node{w}_1), \ldots, x(\node{w}_N))^{\intercal}$ can be inverted on $\mathcal{B}_M$, i.e., if $W$ is a uniqueness (or norming) set for the space $\mathcal{B}_M$, see \cite{Pesenson2008,Pesenson2009}. This notion of uniqueness set is also related to the existence of uncertainty principles on graphs, see \cite{erb2019,erb2019b}. 
\end{remark}  

\section{Examples of Feature-Augmented Classifications} \label{sec:numericalexamples}

In order to obtain a better intuition about the classification behavior of augmented and non-augmented kernels, we test the derived kernel-based classification methods on four data sets. The experimental settings are summarized in Table \ref{tab:experiments}.

\begin{table}[htbp]
 \caption{Experimental setups}
 \label{tab:experiments}
    \centering

   \begin{tabular}{l | cccc}
Name of data set & Two-moon & $\text{\O}$ & WBC & Ionosphere  \\[2mm] \hline  
Type of data set & synthetic & synthetic & \begin{minipage}{3cm} \centering real \tiny (UCI machine \\ learning repository) \end{minipage}& \begin{minipage}{3cm} \centering real \tiny (UCI machine \\ learning repository) \end{minipage}\\[1mm]
\# Nodes & $600$ & $600$ & $683$ & $351$ \\
\# Edges & $10750$ & $8103$ & $232903$ & $61425$ \\[1mm]
Laplacian $\Ll$& \begin{minipage}{3cm} \centering normalized \\ \tiny (standard $\mathbf{A}$) \end{minipage} & \begin{minipage}{3cm} \centering normalized \\ \tiny (standard $\mathbf{A}$) \end{minipage} & \begin{minipage}{3cm} \centering normalized \\ \tiny (weighted $\mathbf{A}$) \end{minipage} & \begin{minipage}{3cm} \centering normalized \\ \tiny (weighted $\mathbf{A}$) \end{minipage}\\[1mm]
GBF & $f_{e^{-50 \mathbf{L}}}$ & $f_{e^{-5 \mathbf{L}}}$ & $f_{e^{-10 \mathbf{L}}}$ & $f_{e^{-5 \mathbf{L}}}$ \\[1mm]
Regularization $\gamma$ & $10^{-4}$ & $10^{-4}$ & $10^{-3}$ & $10^{-3}$ \\[1mm]
1. Feature map $\psi_1$ \quad & \begin{minipage}{3cm} \centering binary \\ \tiny (spectral clustering)\end{minipage} & \begin{minipage}{3cm} \centering similarity \\ \tiny (geometric prior)\end{minipage} & \begin{minipage}{3cm} \centering binary \\ \tiny (spectral clustering)\end{minipage} & \begin{minipage}{3cm} \centering binary \\ \tiny (spectral clustering)\end{minipage} \\
2. Feature map $\psi_2$ \quad & \begin{minipage}{3cm} \centering binary \\ \tiny (given prior)\end{minipage} & \begin{minipage}{3cm} \centering similarity \\ \tiny (geometric prior)\end{minipage} & - & - \\[1mm]
1. Feature kernel & $\Kk^{F_{\bin}}$, $\alpha = -1$  & $\Kk^{F_{\SIM}}$, $\alpha = 10$  & $\Kk^{F_{\bin}}$, $\alpha = -0.5$ & $\Kk^{F_{\bin}}$, $\alpha = -0.5$ \\
2. Feature kernel & $\Kk^{F_{\bin}}$, $\alpha = 0.1$  & $\Kk^{F_{\SIM}}$, $\alpha = 10$  & - & - \\ \hline
\end{tabular}

\end{table}

\subsection{Synthetic two-moon data set}

\begin{figure}[htbp]
	\centering
	\includegraphics[width= 1\textwidth]{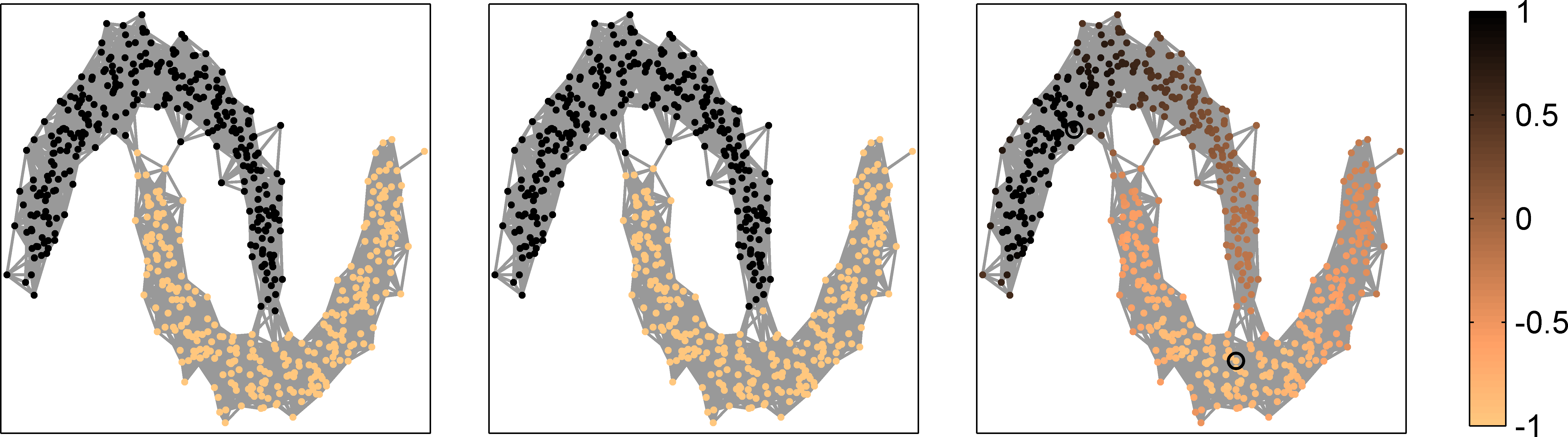}
	\caption{Binary classification on the two-moon data set. Left: the graph with the given labeling. Middle: classification with spectral clustering. Right: GBF-RLS solution $y^*$ for two labeled nodes.}
	\label{fig:twomoon1}
\end{figure}

\begin{figure}[htbp]
	\centering
	\includegraphics[width= 1\textwidth]{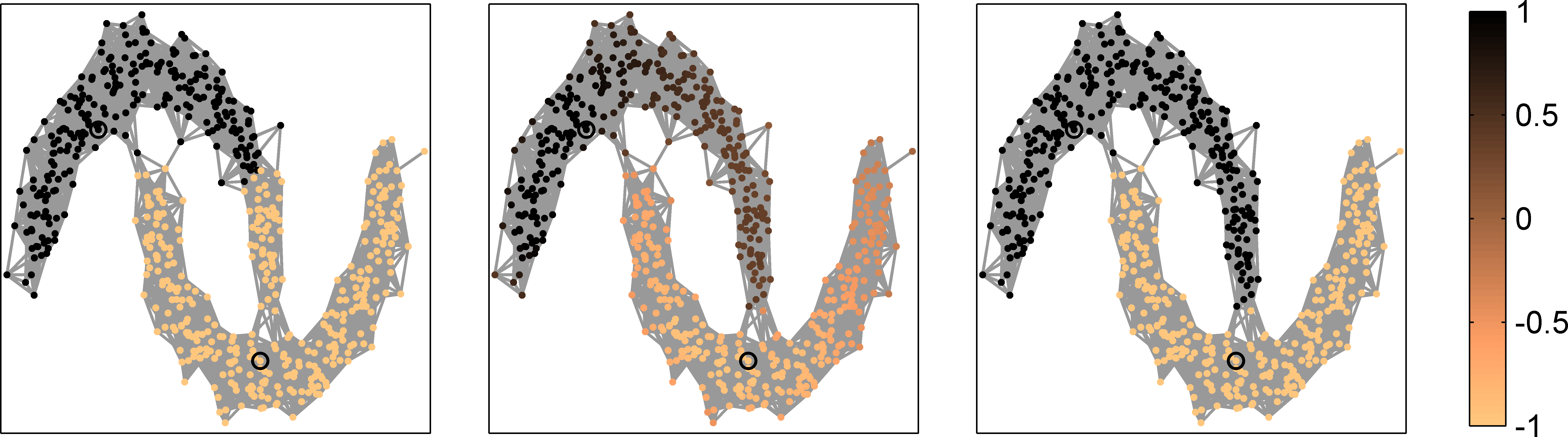}
	\caption{Binary classification on the two-moon data set. Left: GBF-RLS classification corresponding to the solution $y^*$ given in Fig. \ref{fig:twomoon1} (Right). Middle: \psiGBFRLS{} solution $y^*_{\vect{\psi}}$ for the two given labels. Right: \psiGBFRLS{} classification corresponding to the solution $y^*_{\vect{\psi}}$ given in Fig. \ref{fig:twomoon2} (Middle).}
	\label{fig:twomoon2}
\end{figure}

Our first test graph $G_{\twomoon}$ is a synthetic data set referred to as two-moon data set. It is displayed in Fig. \ref{fig:twomoon1} (left). It consists of $n = 600$ vertices in two separated point clouds having the form of half-circles. Each of the two half-circles contains $300$ nodes. Two nodes are connected with an edge, if the euclidean distance between the nodes is smaller than half the radius of the half-circles. As a graph Laplacian for $G_{\twomoon}$ we use the normalized graph Laplacian $\Ll_N$ as constructed in Section \ref{sec:spectralgraphtheory} (3) upon the standard adjacency matrix $\mathbf{A}$.   

To generate the kernel for the supervised GBF-RLS classifier, we use the diffusion GBF $f_{e^{-50 \mathbf{L}}}$ described in Section \ref{sec:examples} (3). As a first additional feature for the \psiGBFRLS{} classifier, we use a binary classification $\psi_{1}$ based on the unsupervised output of a binary spectral clustering algorithm (Shi-Malik normalized cut \cite{shimalik2000}) on the graph $G_{\twomoon}$. The result of this spectral clustering is displayed in Fig. \ref{fig:twomoon1} (middle).

\vspace{1mm}

\noindent{\bfseries Results and Discussion.} 
The results of the supervised GBF-RLS and the semi-supervised \psiGBFRLS{} classifier for $2$ given labels (one in each of the two half-circles) are displayed in Fig. \ref{fig:twomoon1} and \ref{fig:twomoon2}. Due to the separated structure of the two-moon data set, the featured spectral classifier $\psi_{1}$ performs already very well (only two miss-classifications). As only two labeled nodes are given, the GBF-RLS classifier
(Fig. \ref{fig:twomoon2} (left)) is not (yet) able to capture the global structure of the data. On the other hand, the feature-augmented \psiGBFRLS{} classifier 
(Fig. \ref{fig:twomoon2} (right)) contains the information of the auxiliary classificator $\psi_{1}$ and is therefore able to outperform the 
GBF-RLS classifier. In fact, the \psiGBFRLS{} classifier and the spectral classifier
$\psi_{1}$ are identical in this example.

In a second test, we want to see what happens if we use more labeled nodes. The mean accuracies of the classifiers (on $100$ randomly performed experiments) for an increasing number of labels are listed in Table \ref{tab:twomoon}. It is visible that the supervised GBF-RLS classifier improves with a larger number of labels and is able to determine the global classification of the data. In particular, if a larger number of labels is available it is not necessary to use the SSL classifier instead of the supervised one.  
  
\begin{table}[htbp]
 \caption{Mean classification accuracy for two-moon data}
 \label{tab:twomoon}
    \centering

   \begin{tabular}{r*{6}{G}}
  \begin{minipage}{3cm} \centering \# Labeled nodes \end{minipage} & 
  \rz{2} & 
  \rz{4} & 
  \rz{8} &
  \rz{16} & 
  \rz{32} & 
  \rz{64} \smallskip \\ \hline
\begin{minipage}{3cm} \centering Spectral clustering \\ \scriptsize (unsupervised) \end{minipage}
           & 0.9967 & 0.9967 & 0.9967 & 0.9967 & 0.9967 & 0.9967 \\
\begin{minipage}{3cm} \centering GBF-RLS \\ \scriptsize  (supervised) \end{minipage} 
           & 0.7110 & 0.9350 & 0.9860 & 0.9993 & 0.9999 & 1.0000 \\
\begin{minipage}{3cm} \centering \psiGBFRLS{} \\ \scriptsize  (semi-supervised) \end{minipage}
           & 0.9905 & 0.9958 & 0.9959 & 0.9967 & 0.9967 & 0.9967 \\
\end{tabular}

\end{table}

\noindent{\bfseries Refined classification with a second prior.}
As a final experiment on this data set, we consider a second prior $\psi_2$ that separates the vertices of the graph along a vertical axis through the center of the point cloud. We want to see whether the SSL classifier with this additional feature is able to separate the data set as well. Based on four labeled nodes we see in Fig. \ref{fig:twomoon3} that this is in fact the case. Here, the supervised GBF-RLS classifier based on $4$ labels provides already a good separation of the two half-circles while it only has problems to distinguish between left and right hand side data. 

\begin{figure}[htbp]
	\centering
	\includegraphics[width= 1\textwidth]{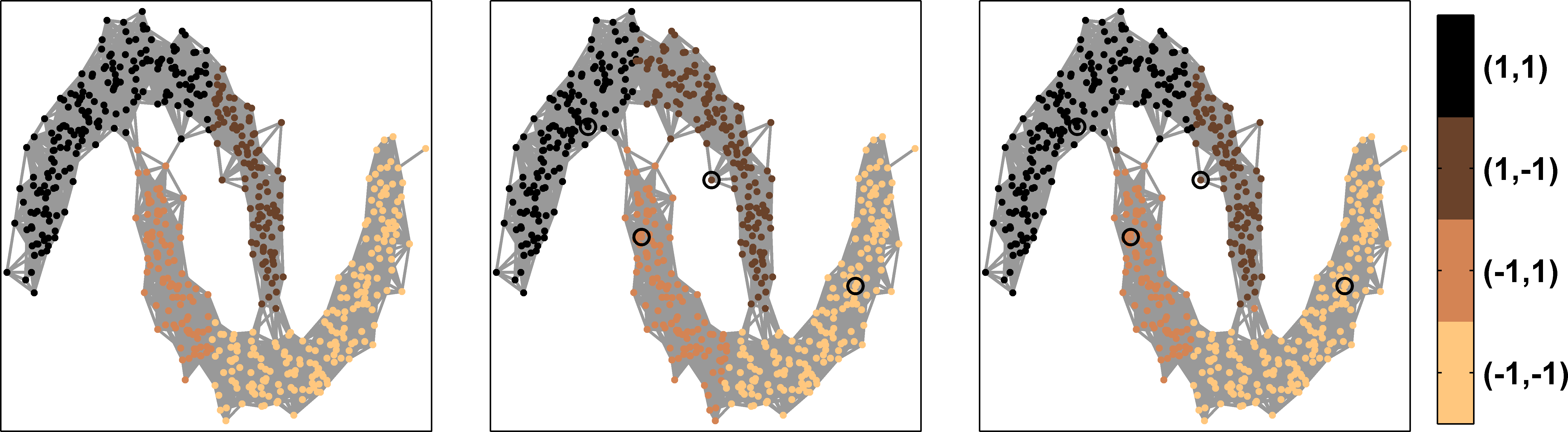}
	\caption{SSL classification of the two moon data set. Left: the two-moon graph with labels in $4$ classes. Middle: supervised classification based on GBF-RLS with $4$ given  labels. Right: \psiGBFRLS{} classification augmented with the two feature maps $\psi_1$ and $\psi_2$.}
	\label{fig:twomoon3}
\end{figure}

\subsection{Synthetic $\text{\O}$ graph}

\begin{figure}[htbp]
	\centering
	\includegraphics[width= 1\textwidth]{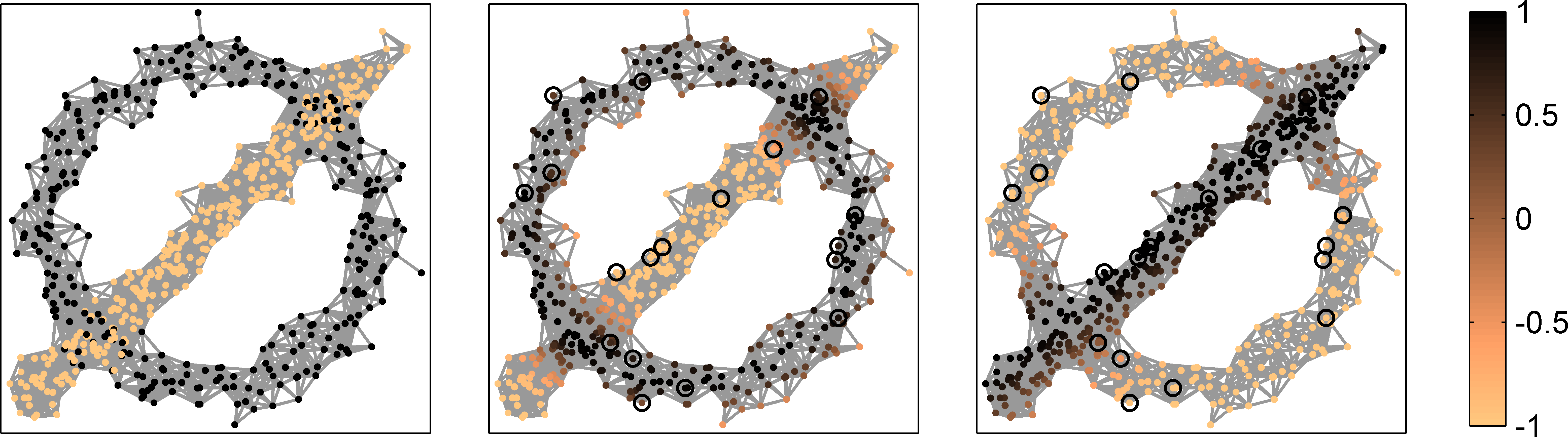}
	\caption{$\text{\O}$ data set. Left: the data set with the given classification of the nodes. Middle: first feature information $\psi_1$. Right: second feature information $\psi_2$.}
	\label{fig:slashedo1}
\end{figure}

Also the second test graph $G_{\text{\O}}$, displayed in Fig. \ref{fig:slashedo1} (left), is a synthetic data set. It consists of $n = 600$ vertices distributed in two point clouds with $300$ nodes: one point cloud in the form of a circle, the other a straight line segment slashing the circle. Two nodes in $G_{\text{\O}}$ are linked if the distance is smaller than $1/5$ of the circle diameter. As for the two-moon graph, we use the normalized graph Laplacian to determine the graph Fourier transform.  

\begin{figure}[htbp]
	\centering
	\includegraphics[width= 1\textwidth]{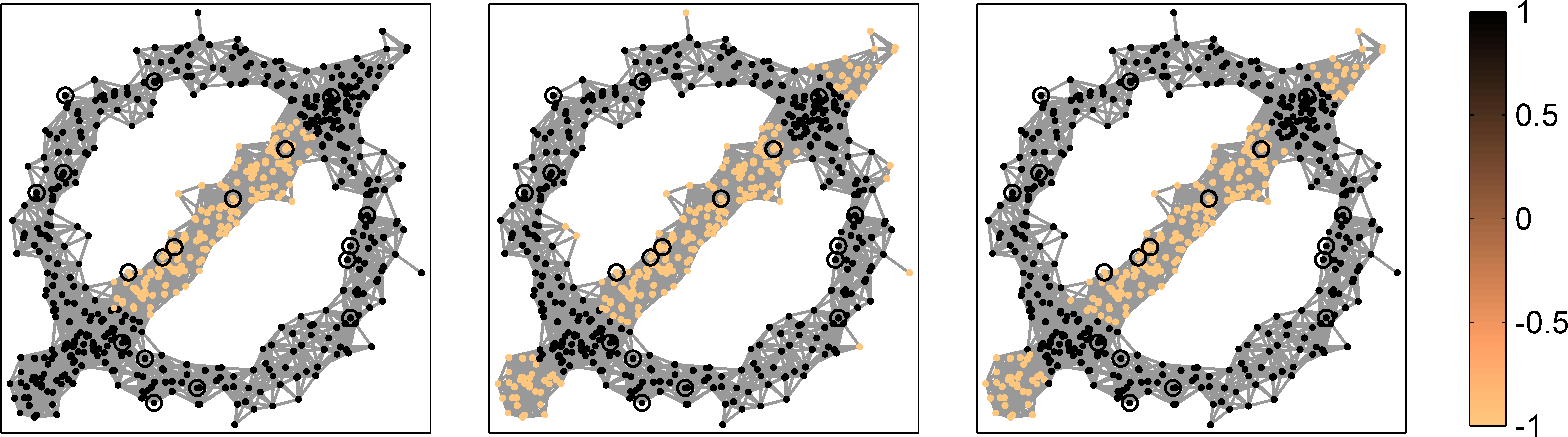}
	\caption{SSL classification of the $\text{\O}$ data set. Left: supervised classification based on GBF-RLS with $18$ given  labels. Middle: \psiGBFRLS{} classification augmented with one additional feature information $\psi_1$. Right: \psiGBFRLS{} classification augmented with the two feature maps $\psi_1$ and $\psi_2$.}
	\label{fig:slashedo2}
\end{figure}

To generate a kernel on $G_{\text{\O}}$ we use the diffusion GBF $f_{e^{-5 \mathbf{L}}}$ described in Section \ref{sec:examples} (3). 
As priors for the semi-supervised \psiGBFRLS{} classifier we use geometric information of the data set. To generate the prior $\psi_{1}$ displayed in Fig. \ref{fig:slashedo1} (middle) we use a distance measure based on the euclidean distances of the single nodes to a reference circle. In a similar way, for the second prior $\psi_{2}$ shown in Fig. \ref{fig:slashedo1} (right) we used the distances of the nodes to a reference line segment.

\vspace{1mm}

\noindent{\bfseries Results and Discussion.} 
In Fig. \ref{fig:slashedo2} the resulting classifications of the supervised (left) and the semi-supervised learning scheme with one (middle) and two (right) additional features are displayed. In all three cases, $18$ labeled nodes were used. Adding the geometric priors, clearly improves the overall performance of the classification and makes the geometric structure of the data more visible. In this case, it makes also sense to augment the kernel with the second prior, as a few misclassifications can be removed. Both priors have however no impact on the classification accuracy at the intersection regions of the circle with the line. In these regions, it is rather random whether a node belongs to the circle or to the line segment. 
Table \ref{tab:slashed-o} confirms these observations and shows that with a small number of labels the two semi-supervised schemes perform better, while with a large number of nodes it is not necessary to augment the GBF with additional priors.

\begin{table}[htbp]
\caption{Mean classification accuracy for $\text{\O}$ data set}
\label{tab:slashed-o}
\centering
\begin{tabular}{r*{6}{G}}
  \begin{minipage}{4cm} \centering \# Labeled nodes \end{minipage} \quad & 
  \rz{2} & 
  \rz{4} & 
  \rz{8} &
  \rz{16} & 
  \rz{32} & 
  \rz{64} \smallskip \\ \hline
\begin{minipage}{4cm} \centering GBF-RLS \\ \scriptsize (supervised) \end{minipage} & 0.5624   & 0.6347 & 0.7511 & 0.8322 & 0.8843 & 0.9152 \\
\begin{minipage}{4cm} \centering \psiGBFRLS{} (1 feature) \\ \scriptsize (semi-supervised) \end{minipage} &0.6164 & 0.6934 & 0.7861 & 0.8260 & 0.8666 & 0.8936 \\
\begin{minipage}{4cm} \centering \psiGBFRLS{} (2 features) \\ \scriptsize (semi-supervised) \end{minipage} &0.6870 & 0.8054 & 0.8723 & 0.8844 & 0.9001 & 0.9101 \\
\end{tabular}

\end{table}

\subsection{Empirical analysis on real data sets}
As a final evaluation, we test our \psiGBFRLS{} schemes on two real world data sets from 
the UCI Machine learning repository. The two data sets are the Ionosphere data set ($351$ nodes) and the (original) Wisconsin Breast Cancer (WBC) data set ($699$ nodes). From these two point clouds, we generate similarity graphs as described in Section \ref{sec:similarity}. As graph Laplacian we use the normalized graph Laplacian based upon the weighted adjacency matrix of the similarity graphs. 
Using the Shi-Malik spectral cut \cite{shimalik2000} as unsupervised classifier, we generate an additional binary feature graph based on the description in Section \ref{sec:examplebinaryclassification}.  

\begin{table}[t]
 \caption{Mean classification accuracy for Wisconsin Breast Cancer data set}
 \label{tab:WBC}
    \centering

   \begin{tabular}{r*{6}{G}}
  \begin{minipage}{3cm} \centering \# Labeled nodes \end{minipage} \quad & 
  \rz{2} & 
  \rz{4} & 
  \rz{8} &
  \rz{16} & 
  \rz{32} & 
  \rz{64} \smallskip \\ \hline
\begin{minipage}{3cm} \centering Spectral clustering \\ \scriptsize (unsupervised) \end{minipage} & 0.9605 & 0.9605 & 0.9605 & 0.9605 & 0.9605 & 0.9605 \\
\begin{minipage}{3cm} \centering GBF-RLS \\ \scriptsize  (supervised) \end{minipage}                       & 0.7273 & 0.8914 & 0.9576 & 0.9635 & 0.9640 & 0.9643 \\
\begin{minipage}{3cm} \centering \psiGBFRLS{} \\ \scriptsize  (semi-supervised) \end{minipage} & 0.9168 & 0.9395 & 0.9554 & 0.9605 & 0.9605 & 0.9605 \\
\end{tabular}
\end{table}

\begin{table}[t]
 \caption{Mean classification accuracy for Ionosphere data set}
 \label{tab:Ionosphere}
    \centering

   \begin{tabular}{r*{6}{G}}
  \begin{minipage}{3cm} \centering \# Labeled nodes \end{minipage} \quad & 
  \rz{10} & 
  \rz{20} & 
  \rz{30} &
  \rz{40} & 
  \rz{50} & 
  \rz{60} \smallskip \\ \hline
\begin{minipage}{3cm} \centering Spectral clustering \\ \scriptsize (unsupervised) \end{minipage}  & 0.7179 & 0.7179 & 0.7179 & 0.7179 & 0.7179 & 0.7179 \\
\begin{minipage}{3cm} \centering GBF-RLS \\ \scriptsize  (supervised) \end{minipage}             & 0.7445 & 0.8049 & 0.8309 & 0.8479 & 0.8607 & 0.8718 \\
\begin{minipage}{3cm} \centering \psiGBFRLS{} \\ \scriptsize  (semi-supervised) \end{minipage} & 0.8025 & 0.8423 & 0.8684 & 0.8813 & 0.8921 & 0.9005
\end{tabular} 

\end{table}

\noindent{\bfseries Results and Discussion.} The mean accuracies (for randomly picked labeled nodes in $100$ tests) for the supervised and 
semi-supervised classifiers are displayed in Table \ref{tab:WBC} and \ref{tab:Ionosphere}. For the Wisconsin Breast Cancer data set the classification results are similar as for the two-moon data set. The unsupervised classifier based on spectral clustering performs already well on the data set and for a small set of labeled nodes the SSL classifier is able to outperform the supervised one. On the other hand, for a large number of labeled nodes it is not necessary to use the feature-augmented kernel.

Interestingly, for the Ionosphere data the behavior of the classifiers looks differently. While the unsupervised spectral classifier is only partially accurate, it apparently contains relevant global information of the data set. In this way, the classification quality of the feature-augmented kernel is solidly superior to the non-augmented kernel also for a larger number of labeled nodes. 


\section*{Acknowledgment}
This work was partially supported by GNCS-In$\delta$AM and by the European Union's Horizon 2020 research and innovation programme ERA-PLANET, grant agreement no. 689443.


\begin{thebibliography}{1}
\scriptsize

\bibitem{Aiolli2015}
{\sc Aiolli, F., Donini, M., Navarin, N. and Sperduti, A.} 
\newblock{ Multiple Graph-Kernel Learning.}
\newblock In: {\em 2015 IEEE Symposium Series on Computational Intelligence}, Cape Town (2012), 1607-1614.

\bibitem{Aronszajn1950}
{\sc Aronszajn, N.} 
\newblock{ Theory of reproducing kernels.}
\newblock {\em Trans. Amer. Math. Soc. 68} (1950), 337–404.

\bibitem{Belkin2006}
{\sc Belkin, M., Niyogi, P. and Sindhwani, V.} 
\newblock {Manifold Regularization: A Geometric Framework for Learning from Labeled and Unlabeled Examples.}
\newblock {\em J. Mach. Learn. Res. 7} (2006), 2399--2434.

\bibitem{BelkinNiyogi2004}
{\sc Belkin, M. and Niyogi, P. } 
\newblock{ Semi-supervised learning on Riemannian manifolds.}
\newblock {\em Machine Learning 56}, 1-3 (2004), 209--239.

\bibitem{BelkinMatveevaNiyogi2004}
{\sc Belkin, M., Matveeva I., and Niyogi, P. } 
\newblock{ Regularization and Semi-supervised Learning on Large Graphs.}
\newblock In: {\em Shawe-Taylor, J., Singer, Y. (Eds.): Learning Theory, COLT 2004, LNAI 3120}, Springer, Berlin, Heidelberg (2004), 624–-638.

\bibitem{Berman1994}
{\sc Berman, A. and Plemmons, J.}
\newblock {\em Nonnegative Matrices in the Mathematical Sciences }.
\newblock {SIAM}, Philadelphia, 1994.

\bibitem{Bozzini1}
\textsc{Bozzini, M., Lenarduzzi, L., Rossini, M. and Schaback, R.},
\newblock{Interpolation with variably scaled kernels.}
\newblock {\em IMA J. Numer. Anal. 35} (2015), 199--219.

\bibitem{campi2019}
{\sc Campi, C., Marchetti, F. and Perracchione, E.}
\newblock  Learning via Variably Scaled Kernels.
\newblock {\em Preprint} (2019).

\bibitem{Chung}
{\sc Chung, F.R.K.}
\newblock {\em Spectral Graph Theory}.
\newblock {American Mathematical Society}, Providence, RI, 1997.

\bibitem{demarchi2019a}
{\sc De Marchi, S., Erb, W., Marchetti, F., Perracchione, E. and Rossini, M.}
\newblock  Shape-Driven Interpolation with Discontinuous Kernels: Error Analysis, Edge Extraction and Applications in Magnetic Particle Imaging. 
\newblock {\em SIAM J. Sci. Comput. (in press)} (2020) DOI:10.1137/19M1248777.

\bibitem{demarchi2019b}
{\sc De Marchi, S., Marchetti, F. and Perracchione, E.}
\newblock  Jumping with variably scaled discontinuous kernels (VSDKs). 
\newblock {\em Bit Numer. Math. (in press)} (2019), DOI:10.1007/s10543-019-00786-z.

\bibitem{erb2019}
{\sc Erb, W.}
\newblock {Shapes of Uncertainty in Spectral Graph Theory}.
\newblock {\em 	arXiv:1909.10865 \/} (2019).

\bibitem{erb2019b}
{\sc Erb, W.}
\newblock {Graph Signal Interpolation with Positive Definite Graph Basis Functions}.
\newblock {\em 	arXiv:1912.02069 \/} (2019).

\bibitem{Gaertner2006}
{\sc G{\"a}rtner, T., Le, Q.V., and Smola, A.J.}
\newblock {A Short Tour of Kernel Methods for Graphs}.
\newblock Preprint (2006).

\bibitem{GodsilRoyle2001}
{\sc Godsil, C., and Royle, G.}
\newblock {\em Algebraic Graph Theory.}
\newblock Springer-Verlag, New York, 2001.

\bibitem{goenen2011}
{\sc G{\"o}nen M, W., and Alpaydin, E.}
\newblock {Multiple Kernel Learning Algorithms}.
\newblock {\em 	J. Machine Learning Research 12} (2011), 2211--2268.

\bibitem{Hammack2011}
{\sc Hammack, R., Imrich, W., and Klavzar, S.}
\newblock {\em Handbook of Product Graphs.} 
\newblock {2nd Edition, CRC Press, Inc.}, NW Boca Raton, FL, 2011.

\bibitem{Joachims1999}
{\sc Joachims, T.}
\newblock{Transductive Inference for Text Classiffication using Support Vector
Machines.}
\newblock{\em Proceedings of ICML-99} (1999), 200--209.

\bibitem{HornJohnson1985}
{\sc Horn, R.A., and Johnson, C.R.}
\newblock {\em Matrix Analysis},
\newblock Cambridge University Press, 1985.

\bibitem{KondorLafferty2002} 
{\sc Kondor, R.I., and Lafferty, J.}
\newblock {Diffusion kernels on graphs and other discrete input spaces.}
\newblock in {\em Proc. of the 19th. Intern. Conf. on Machine Learning ICML02\/} (2002), 315-322.

\bibitem{Mhaskar2019}
{\sc Mhaskar, H., Pereverzyev, S.V., Semenov, V.Y. and Semenova, E.V.}
\newblock {Data Based Construction of Kernels for Semi-Supervised Learning With Less Labels.}
\newblock {Frontiers in Applied Mathematics and Statistics 5}, (2019).

\bibitem{Nigam2000} 
{\sc Nigam, K., Mccallum, A.K., Thrun, S., and Mitchell, T. }
\newblock {Text Classification from Labeled and Unlabeled Documents using EM.}
\newblock {\em Machine Learning 39} (2000) 103--134.

\bibitem{Pesenson2008}
{\sc Pesenson, I.Z.}
\newblock {Sampling in Paley-Wiener spaces on combinatorial graphs.} 
\newblock {\em Trans. Amer. Math. Soc. 360}, 10 (2008), 5603–-5627.

\bibitem{Pesenson2009}
{\sc Pesenson, I.Z.}
\newblock {Variational Splines and Paley-Wiener Spaces on Combinatorial Graphs.} 
\newblock {\em Constr. Approx. 29}, 1 (2009), 1--21.

\bibitem{Rifkin2003}
{\sc Rifkin, R., Yeo, G., and Poggio, T.}
\newblock {Regularized least-squares classification.} 
\newblock In {\em Nato Science Series Sub Series III Computer and SystemsSciences, vol. 190}, (2003), 131–-154.

\bibitem{Romero2017}
{\sc Romero, D., Ma, M., and Giannakis, G.B.}
\newblock {Kernel-Based Reconstruction of Graph Signals.} 
\newblock {\em IEEE Transactions on Signal Processing 65}, 3 (2017), 764--778.

\bibitem{Saitoh2005}
{\sc Saitoh, S.} 
\newblock {Best approximation, Tikhonov regularization and reproducing kernels.}
\newblock {\em Kodai Math. J. 28}, (2005), 359--367.

\bibitem{Scha98} 
{\sc Schaback, R.}
\newblock {Native Hilbert spaces for radial basis functions I.}
\newblock In {\em M. W. M\"uller et. al., eds., New Developments in Approximation Theory. 2nd International Dortmund Meeting (IDoMAT '98), vol. 132}, Int. Ser. Numer. Math., Birh\"auser Verlag, Basel (1999), 255--282.

\bibitem{schaeffer2007}
{\sc Schaeffer, S.E.}
\newblock {Graph clustering.}
\newblock {\em Computer Science Review 1} (2007), 27--64.

\bibitem{Schoelkopf2002} 
{\sc Sch\"olkopf, B. and Smola, A.}
\newblock {\em Learning with Kernels.}
MIT Press, Cambridge, 2002.

\bibitem{shimalik2000}
{\sc Shi, J., Malik, J.}
\newblock {Normalized cuts and image segmentation.}
\newblock {\em IEEE Trans. Pattern Anal. Mach. Intell. 22}, 8 (2000), 888--901.

\bibitem{shuman2016}
{\sc Shuman, D.I., Ricaud, B., and Vandergheynst, P.}
\newblock {Vertex-frequency analysis on graphs.} 
\newblock {\em Appl. Comput. Harm. Anal. 40}, 2 (2016), 260--291.

\bibitem{SmolaKondor2003}
{\sc Smola, A. and Kondor. R.}
\newblock {Kernels and Regularization on Graphs.} 
\newblock In {\em Learning Theory and Kernel Machines}, Springer Berlin Heidelberg (2003), 144--158.

\bibitem{StankovicDakovicSejdic2019}
{\sc Stankovi{\'c}, L., Dakovi{\'c}, L., and Sejdi{\'c}, E.}
\newblock {Introduction to Graph Signal Processing.} 
\newblock In {\em Vertex-Frequency Analysis of Graph Signals}, Springer, (2019), 3--108.

\bibitem{Vapnik1998}
{\sc Vapnik, V.N},
\newblock{\em Statistical learning theory.}
\newblock{Wiley}, New York (1998).

\bibitem{Wahba1990}
{\sc Wahba, G.},
\newblock{\em Spline Models for Observational Data.}
\newblock{CBMS-NSF Regional Conference Series in Applied Mathematics, volume 59, SIAM, Philadelphia}, (2018).

\bibitem{Wang2012}
{\sc Wang, S., Huang, Q., Jiang, S. and Tian, Q.},
\newblock{${\rm S}^{3}{\rm MKL}$: Scalable Semi-Supervised Multiple Kernel Learning for Real-World Image Applications.}
\newblock{\em IEEE Transactions on Multimedia 14}, 4 (2012), 1259-1274.

\bibitem{Ward2018interpolating}
{\sc Ward, J.P., Narcowich, F.J., and Ward, J.D.},
\newblock{Interpolating splines on graphs for data science applications.}
\newblock{arXiv:1806.10695 (math.NA)} (2018).

\bibitem{We05} 
{\sc Wendland, H.}
\newblock {\em Scattered Data Approximation.}
\newblock Cambridge University Press, Cambridge, 2005.

\bibitem{Yan2010}
{\sc Yan, F., Mikolajczyk, K., Kittler, J., and Tahir, M.},
\newblock{Combining Multiple Kernels by Augmenting the Kernel Matrix.}
\newblock In {\em Multiple Classifier Systems, 9th International Workshop}, Cairo, Egypt (2010), 175-184.

\bibitem{Zhu05} 
{\sc Zhu, X.}
\newblock {\em Semi-Supervised Learning with Graphs.}
\newblock PhD thesis, Carnegie Mellon University, 2005.


\end{thebibliography}
\end{document}